\documentclass{article}

\PassOptionsToPackage{numbers, sort&compress}{natbib}

\usepackage[preprint]{neurips_2026}
\makeatletter
\renewcommand{\@noticestring}{\@neuripsordinal\ Conference on Neural Information Processing Systems (NeurIPS \@neuripsyear).}
\makeatother

\usepackage[utf8]{inputenc} % (ok; often unnecessary on newer LaTeX)
\usepackage[T1]{fontenc}
\usepackage{times}          % NeurIPS requires Times New Roman
\usepackage{microtype}
\usepackage{xcolor}
\usepackage{url}
\usepackage{hyperref}

\hypersetup{
  colorlinks=true,
  linkcolor=blue,
  citecolor=blue,
  urlcolor=blue
}

\usepackage{graphicx}       % recommended for \includegraphics
\usepackage{booktabs}       % recommended for professional tables
\usepackage{multirow}
\usepackage{array}
\usepackage{float}
\usepackage{tikz}
\usetikzlibrary{positioning, arrows.meta, decorations.pathreplacing, calc, fit, backgrounds}
\usepackage{pgfplots}
\pgfplotsset{compat=1.18}

\usepackage{subcaption}
\usepackage{amsmath,amssymb,amsfonts} % AMS math
\usepackage{mathtools}               % extends amsmath
\usepackage{bm}                      % bold math
\usepackage{amsthm}
\newtheorem{proposition}{Proposition}
\newtheorem{remark}{Remark}
\usepackage{algorithm}
\usepackage{algorithmic}

\usepackage{enumitem}
\usepackage{comment}
\usepackage{titletoc}       % local (appendix-scoped) clickable TOC

\title{Learning Scenario Reduction for Two-Stage Robust Optimization with Discrete Uncertainty}

\author{
\textbf{Tianjue Lin$^{1}$} \quad
\textbf{Jianan Zhou}$^{1}\thanks{Corresponding author.}$ \quad
\textbf{Jieyi Bi$^{1}$}\\
\textbf{Yaoxin Wu}$^{2}\footnotemark[1]$ \quad
\textbf{Wen Song$^{3}$} \quad
\textbf{Zhiguang Cao$^4$} \quad
\textbf{Jie Zhang}$^1$ \vspace{2mm}\\
$^1$Nanyang Technological University\quad
$^2$Eindhoven University of Technology\\
$^3$Shandong University\quad
$^4$Singapore Management University \vspace{2mm}\\
\textnormal{tianjue002@ntu.edu.sg},\, 
\textnormal{jianan004@ntu.edu.sg},\, 
\textnormal{jieyi001@ntu.edu.sg},\\
\textnormal{y.wu2@tue.nl},\, 
\textnormal{wensong@email.sdu.edu.cn},\, 
\textnormal{zgcao@smu.edu.sg},\, 
\textnormal{zhangj@ntu.edu.sg}
}

\begin{document}
\raggedbottom

\maketitle

\begin{abstract}
Two-Stage Robust Optimization (2RO) with discrete uncertainty is challenging, often rendering exact solutions prohibitive. Scenario reduction alleviates this issue by selecting a small, representative subset of scenarios to enable tractable computation. However, existing methods are largely \emph{problem-agnostic}, operating solely on the uncertainty set without consulting the feasible region or recourse structure. In this paper, we introduce PRISE, a \emph{problem-driven} sequential lookahead heuristic that constructs reduced scenario sets by evaluating the marginal impact of each scenario. While PRISE yields high-quality scenario subsets, each selection step requires solving multiple subproblems, making it computationally expensive at scale. To address this, we propose NeurPRISE, a neural surrogate model built on a GNN-Transformer backbone that encodes the per-scenario structure via graph convolution and captures cross-scenario interactions through attention. NeurPRISE is trained via imitation learning with a gain-aware ranking objective, which distills marginal gain information from PRISE into a learned scoring function for scenario ranking and selection. 
Extensive results on three 2RO problems show that NeurPRISE consistently achieves competitive regret relative to comprehensive methods, maintains strong scalability with varying numbers of scenarios, and delivers $17$--$200\times$ speedup over PRISE. NeurPRISE also exhibits strong zero-shot generalization, effectively handling instances with larger problem scales (up to $5\times$), more scenarios (up to $4\times$), and distribution shifts.
\end{abstract}
% and the difficulty is further exacerbated when decision variables are combinatorial
% distilling PRISE's marginal gain information into a learned scoring function. 
% we propose \emph{NeurPRISE}, a GNN--Transformer architecture that encodes per-scenario MIP structure on a bipartite constraint-variable graph and models cross-scenario interactions via Transformer. 
% achieving up to $14\times$ lower optimality gap than Neur2RO, 
% (Selection, Vertex Cover, and Capacitated Facility Location)    
\section{Introduction} \label{sec:intro}

% Optimization traditionally assumes that all problem parameters are fixed and known. In practice, however, this assumption rarely holds: parameters are inherently uncertain. Customer demand in supply chains and energy systems fluctuates unpredictably across time and geography. Measurement errors and estimation inaccuracies mean that input data (e.g., costs, capacities, or travel times) are only approximations of their true values.
% This uncertainty has real consequences. A plan optimized for today's demand forecast may become costly or infeasible when actual demand diverges. Solutions calibrated to approximate data can violate constraints or perform highly suboptimally once deployed under true parameter values \cite{ben2000robust}. Robust Optimization (RO) addresses this by seeking solutions that remain feasible under worst-case parameter realizations within a predefined uncertainty set \cite{ben2000robust}. 

Deterministic optimization paradigms typically assume that all problem parameters are precisely known at the time of decision-making. In practice, however, this assumption rarely holds. Parameters are inherently uncertain, often subject to fluctuating customer demand, measurement inaccuracies, or estimation errors. Neglecting these uncertainties can have severe consequences: solutions deemed "optimal" under estimated parameters are often fragile, leading to constraint violations or catastrophic performance degradation when deployed in real-world environments. Robust Optimization (RO) mitigates this risk by seeking solutions that remain feasible and performant under the worst-case parameter realizations within a predefined uncertainty set \cite{ben2000robust, ben2009robust, bertsimas2004price}. Two-Stage Robust Optimization (2RO) extends this paradigm by allowing recourse decisions to be deferred until after the uncertainty is observed \cite{ben2004adjustable, zeng2013solving, zhao2012exact}, effectively balancing proactive planning with adaptive flexibility.

While most literature focuses on continuous uncertainty models (e.g., box \cite{soyster1973convex}, polyhedral \cite{bertsimas2006robust}, or ellipsoidal \cite{ben2000robust} sets), discrete uncertainty sets are often more representative of real-world settings where decision-makers rely on finite collections of historical scenarios or expert-defined scenarios \cite{goerigk2024data}.
However, 2RO under discrete uncertainty, especially when coupled with discrete decision variables, is notoriously difficult. The interplay between first-stage commitments and second-stage recourse renders exact solution methods (e.g., Column-and-Constraint Generation (CCG)~\cite{zeng2013solving, zhao2012exact}) prohibitive as the scenario space scales \cite{kasperski2016robust}, creating a critical need for efficient solution strategies.
% This work addresses the intersection of \emph{2RO within discrete decision spaces and discrete uncertainty sets}. This class of problems is notoriously computationally intensive, as it requires navigating the combinatorial complexities inherent to both the decision variables and the uncertainty structure \cite{kasperski2016robust}.
% This work focuses on 2RO under discrete decision spaces and uncertainty sets—a class of problems that is inherently computationally challenging due to the combinatorial nature of both \cite{kasperski2016robust}.

Traditional heuristic approaches, including $k$-adaptability \cite{bertsimas2010finite, hanasusanto2015k}, decision rules \cite{bertsimas2018binary}, and uncertainty set splitting \cite{postek2016multistage}, typically gain tractability by imposing structural restrictions on the solution space, often at the expense of optimality. Alternatively, scenario reduction simplifies the problem at the input level: it compresses the uncertainty set into a smaller, representative subset to enable tractable exact solving.
% Scenario reduction takes a different approach: rather than modifying the solution method, it reduces the uncertainty set to a smaller, representative subset, enabling tractable exact solving downstream.
However, most existing reduction methods remain \emph{problem-agnostic} \cite{goerigk2019representative, goerigk2023optimal}, operating exclusively on the uncertainty set while ignoring the feasible region and recourse dynamics.
While \cite{fairbrother2024problem} introduce a problem-driven approach for single-stage RO by formulating scenario reduction as a Mixed-Integer Linear Programming (MILP), this formulation does not readily extend to 2RO, as it may break into a much harder nested (e.g., tri-level) optimization problem.
% where evaluating the marginal impact of a scenario requires solving the recourse problem.}
% the marginal impact of a scenario can only be determined by solving a costly recourse optimization problem.
% extending this logic to the two-stage setting is non-trivial. In 2RO, evaluating a scenario’s marginal impact requires solving the recourse subproblem—a step that is itself computationally intensive.
% \cite{fairbrother2024problem} introduce a problem-driven criterion for single-stage RO, but it does not carry over to the two-stage setting, where evaluating a scenario's cost requires solving the recourse problem.

In this paper, we study 2RO under discrete uncertainty. We introduce PRoblem-drIven SEquential lookahead (PRISE), a heuristic that frames scenario reduction as an objective-aware sequential decision process. This framework constructs the reduced scenario set incrementally by evaluating the \emph{marginal gain} associated with each candidate scenario. Specifically, PRISE identifies the scenario that, when integrated into the current partial set, yields the most significant improvement in the approximation of the robust objective. By the \emph{min-max-min} structure of 2RO, the reduced-scenario objective value is \emph{monotonically non-decreasing} as scenarios are added. Empirically, PRISE exhibits strong \emph{diminishing returns}: the most critical worst-case scenarios are identified in the first few iterations, so this incremental construction yields a nested sequence of subsets that is near-optimal across all cardinalities. Our results show that PRISE converges to within 1\% of the full-scenario objective using 2--8\% of the scenario set, whereas problem-agnostic baselines need at least $2$--$4{\times}$ as many scenarios for equivalent coverage (Appendix~\ref{subsec:prise_empirical_analysis}).

While PRISE yields high-quality scenario subsets, it becomes computationally prohibitive as the selection horizon expands, as each selection step requires solving the reduced problem for every candidate scenario subset. To address this, we propose NeurPRISE, a neural surrogate model designed to accelerate the selection process.
NeurPRISE leverages a GNN-Transformer backbone: a GNN encodes each scenario structure (e.g., MILP) onto a bipartite constraint-variable graph via message passing, while Transformer blocks capture cross-scenario interactions through attention.
% that encodes each scenario's MIP structure on a bipartite constraint-variable graph via message passing and models cross-scenario interactions via Transformer attention.
The model is trained via Imitation Learning (IL) using a gain-aware ranking objective. This objective distills the marginal gain information from PRISE into a learned scoring function, enabling NeurPRISE to rank scenarios by their expected contribution to the robust objective.
% At inference, scenario selection requires only a GNN forward pass with no additional optimization solves. 
Empirical results show that this learned policy dramatically accelerates scenario reduction while maintaining solution quality comparable to the original PRISE heuristic.
% The learned policy dramatically accelerates scenario reduction while maintaining solution quality comparable to PRISE.
% The core novelty lies in the training signal---not any single neural component: ablation studies (Appendix~\ref{subsec:arch_loss_ablation}) show that replacing SWKL with MSE regression degrades by up to $2\times$, while replacing the GNN-Transformer with DeepSets degrades by $1.4$--$13.0\times$ (strongest on SEL where combinatorial structure is critical).
Our contributions are summarized as follows:
\begin{itemize}[leftmargin=*]
    \item
    % We study a challenging yet realistic setting of 2RO under discrete uncertainty and decision spaces. Accordingly, we introduce PRISE, a problem-driven sequential lookahead heuristic for scenario reduction, with established monotonicity (Proposition~\ref{prop:prise_monotonicity}) and empirically demonstrate strong diminishing returns.
    We investigate 2RO within the challenging yet practical regime of discrete uncertainty. To address this, we introduce PRISE, a problem-driven sequential lookahead heuristic for effective scenario reduction. We theoretically prove its monotonicity and provide empirical evidence of its strong diminishing returns.
    % demonstrating that a small number of iterations effectively captures the critical uncertainty.
    
    \item 
    % To mitigate the expansive cost of PRISE, we propose a neural surrogate model NeurPRISE, which leverages a GNN-Transformer backbone to captures both scenario-level structure and set-level interactions. NeurPRISE is trained using imiration leanring with a gain-aware ranking objective that distills PRISE's marginal-gain information into a learned scoring function.
    To mitigate the computational overhead of PRISE, we propose NeurPRISE, a neural surrogate model that leverages a GNN-Transformer backbone to capture both scenario-level structure and set-level interactions. NeurPRISE is trained via IL using a gain-aware ranking objective, allowing it to rank scenarios by their expected contribution to the robust objective.
    % bypassing the need for expensive optimization solves during selection.
    
    \item Experiments across three 2RO problem classes (Selection, Vertex Cover, and Capacitated Facility Location) show that NeurPRISE is reliably competitive with classical and ML-based baselines, approaching PRISE-oracle quality at a fraction of its cost. Furthermore, NeurPRISE exhibits strong generalization to larger problem scales, higher scenario counts, and distribution shifts.
    % Experiments on three 2RO problem classes, including selection, vertex cover, and capacitated facility location, show that NeurPRISE consistently achieves lower optimality gaps than classical and ML-based baselines, while approaching the solution quality of the PRISE oracle at a fraction of its computational cost. Moreover, NeurPRISE demonstrate strong zero-shot generalization on settings of larger scales, higher scenario counts, and distributional shifts.
    % Models trained on small instances generalize to larger scales (up to $5\times$), higher scenario counts, and distributional shifts without retraining.
\end{itemize}

% \textbf{PRISE Heuristic and Analysis:}
% \textbf{NeurPRISE --- Neural Imitation of PRISE:}
% \textbf{Empirical Validation and Generalization:}    
\section{Related Work}

\textbf{Two-Stage Robust Optimization.}
% Robust optimization \citep{ben2009robust} ensures feasibility under worst-case realizations. 
Robust optimization \citep{ben2009robust} seeks solutions that remain feasible and performant under the worst-case realizations of uncertainty.
% whose sets range from interval and ellipsoidal \citep{ben2009robust} to polyhedral \citep{bertsimas2006robust}, budget \citep{bertsimas2004price}, and discrete sets \citep{kasperski2016robust}.
Two-stage robust optimization generalizes this framework by allowing for recourse after uncertainty is revealed \citep{ben2004adjustable}, albeit at a significantly higher computational cost.
To achieve exact 2RO solutions, CCG \citep{zhao2012exact, zeng2013solving} provides a framework that avoids the full extensive-form reformulation, which otherwise necessitates a block of recourse variables and constraints for every possible scenario.
Instead, CCG iteratively identifies the current worst-case scenario and appends its associated recourse block to the master problem, terminating when the robust objective can no longer be improved.
To sidestep the complexity of exact methods, various heuristics have been developed.
% Heuristic approaches are motivated by the need to sidestep this complexity: 
Decision-rule approaches \citep{bertsimas2015nearoptimal, bertsimas2018binary} are widely used for continuous uncertainty, where restricting recourse to a structured function of the uncertainty often yields a tractable reformulation. 
$K$-adaptability \citep{bertsimas2010finite, hanasusanto2015k, subramanyam2020k} pre-commits to $K$ recourse plans before uncertainty is revealed, effectively diversifying recourse policies to hedge against both continuous and discrete uncertainty.
% Our work is complementary to these efforts: whereas $K$-adaptability focuses on diversifying recourse, we prioritize identifying a $K$-cardinality subset of representative scenarios to solve for a single, optimal first-stage decision.
Recently, ML approaches have emerged to further accelerate these processes. For instance,
\cite{bertsimas2024machine} train a model to predict solution strategies, namely the first-stage decision, the worst-case scenario, and the active constraint set. However, their approach requires continuous recourse and polyhedral uncertainty. 
% \cite{xiao2024neural} study robust routing under budgeted travel-time uncertainty, learning attention-based policies for the min-max regret objective; 
\cite{brenner2025a} propose a data-driven framework that leverages a deep generative model to produce less conservative uncertainty realizations, but their method requires continuous uncertainty for their gradient-based approach.
\cite{goerigk2024data} use ML to identify critical initial scenarios for CCG, targeting solver warm-start rather than scenario set reduction. \cite{dumouchelle2023neur2ro} learn a surrogate for the full two-stage objective as a function of the first-stage decision and uncertainty realization, embedding it within an MILP solver to accelerate worst-case search.
% Our approach is complementary: rather than diversifying recourse, we select $K$ representative scenarios to solve for a single optimal first-stage decision.

\textbf{Scenario Reduction.}
Scenario reduction seeks to construct a compact yet representative subset of scenarios to alleviate computational burdens.
% Scenario reduction constructs a smaller, representative subset of scenarios to mitigate computational complexity.
In stochastic programming, classical approaches focus on minimizing probability distance metrics \citep{dupavcova2003scenario, heitsch2003scenario}, while more recent methods integrate objective sensitivity and constraint structure to enhance decision quality \citep{keutchayan2023problem, henrion2022problem, hewitt2022decision, bertsimas2023optimization}.
% In stochastic programming, classical methods minimize probability distance metrics \citep{dupavcova2003scenario, heitsch2003scenario} or incorporate objective sensitivity and constraint structure \citep{keutchayan2023problem, henrion2022problem, hewitt2022decision, bertsimas2023optimization}. 
In Two-Stage Stochastic Programming (2SP), \cite{wu2022learning} learn latent representations to facilitate scenario reduction, while \cite{pmlr-v235-wu24ag} select representative scenarios using a first-stage solution consistency reward. However, this reward is tailored to the expected-value objective in 2SP and cannot extend to the worst-case setting of 2RO.
% this reward is tailored to 2SP's expected-value objective and does not directly transfer to 2RO's worst-case setting.
Within the realm of robust optimization, the most relevant prior work is \cite{goerigk2023optimal}, who investigates scenario reduction for 2RO by solving a combinatorial selection problem with provable approximation guarantees.
% The closest prior work in robust optimization is \cite{goerigk2023optimal}, who study scenario reduction for 2RO with discrete uncertainty by solving a combinatorial selection problem with provable objective approximation guarantees; 
% we use their method as the SOR baseline. 
However, their formulation is problem-agnostic, relying solely on the geometry of the uncertainty set without exploiting the feasible region or recourse structure.
% However, their formulation depends only on the uncertainty set and does not exploit the feasible region or recourse structure, making it \emph{problem-agnostic}.
While \cite{fairbrother2024problem} propose a problem-driven criterion for single-stage RO with objective uncertainty, their criterion is designed for the single-stage setting and does not extend to 2RO, where the objective depends on the recourse decision.
% \cite{fairbrother2024problem} propose a problem-driven criterion for single-stage RO with objective uncertainty, but their approach relies on the linear structure of the single-stage objective and does not extend to the two-stage setting, where the objective also depends on the recourse decision.
Our work bridges this gap by extending the problem-driven paradigm to 2RO through a learning-based approach that directly evaluates the marginal impact of each scenario.
% Our work extends the problem-driven perspective to 2RO through a learning-based approach that directly evaluates marginal scenario impact on the robust objective.
% \paragraph{ML for Combinatorial Optimization.}
% Neural networks have been applied to accelerate combinatorial optimization \citep{bengio2021machine}, including supervised learning of branching rules \citep{gasse2019exact}, cut selection \citep{paulus2022learning}, and column generation pricing \citep{morabit2021machine}, as well as RL-based policies \citep{bello2016neural, kool2018attention}.

\section{Preliminaries}
\label{sec:preliminaries}

% \colorb: entire Section 3 rewritten 2026-04-12 (moved MILPs to Appendix C.1)

In this paper, we focus on linear two-stage robust optimization over a finite discrete uncertainty set $\Xi = \{\xi_1, \dots, \xi_S\}$ of size $S$, supplied by the decision-maker (e.g., derived from historical scenarios). Following the notation of \cite{zeng2013solving}, the problem is formulated as follows:
\begin{equation}
\begin{aligned}
\min_{x \in X} \quad & c^\top x + \max_{\xi \in \Xi} \min_{y \in F(x, \xi)} b_\xi^\top y, \\
\text{s.t.} \quad & Ax \geq d,
\end{aligned}
\label{eq:2ro_formulation}
\end{equation}
where $x \in \mathbb{R}^n$ and $y \in \mathbb{R}^m$ denote the first- and second-stage decisions; $c \in \mathbb{R}^n$ and $b_\xi \in \mathbb{R}^m$ are cost vectors; and $F(x, \xi) = \{y \in Y : Gy \geq h - Ex - M\xi\}$ is the recourse feasible set. We consider integer first-stage decisions ($X \subseteq \mathbb{Z}^n_+$) with integer or continuous recourse ($Y \subseteq \mathbb{Z}^m_+$ or $\mathbb{R}^m_+$), and assume \emph{relatively complete recourse}: $F(x, \xi) \neq \emptyset$ for all feasible $x$ and $\xi \in \Xi$.
The remaining parameters are $A \in \mathbb{R}^{p \times n}$, $d \in \mathbb{R}^p$, $G \in \mathbb{R}^{r \times m}$, $h \in \mathbb{R}^r$, $E \in \mathbb{R}^{r \times n}$, and $M \in \mathbb{R}^{r \times q}$.

\textbf{Core Notation.}
Let $Q(x,\xi) := \min_{y \in F(x,\xi)} b_\xi^\top y$ denote the \emph{recourse cost} of first-stage decision $x$ under scenario $\xi$. For any first-stage decision $x$, we denote its \emph{full-scenario cost} as
\[
Z(x) := c^\top x + \max_{\xi \in \Xi} Q(x,\xi).
\]
For any scenario subset $R \subseteq \Xi$, we define the \emph{restricted-scenario objective value} $V(R)$ as follows:
\begin{equation}
\label{eq:V_def}
V(R) \;:=\; \min_{x \in X} \left[ c^\top x \,+\, \max_{\xi \in R}\, Q(x,\xi) \right],
\end{equation}
with the convention $V(\emptyset) := 0$. 
The value $V(\Xi)$ corresponds to the \emph{full-scenario objective value}. 
When $R=R^{(k)}$ with $|R^{(k)}|=k$, $V(R^{(k)})$ is the reduced-scenario objective value used to evaluate a $k$-scenario reduction.
To evaluate the quality of a reduced scenario set $R^{(k)}$, let $x^{(k)\star}$ denote an optimal first-stage solution to the reduced problem in Eq.~\eqref{eq:V_def}. The regret metric measures the suboptimality of $x^{(k)\star}$ relative to the true full-scenario optimum~\cite{dumouchelle2023neur2ro}:
\begin{equation}
\label{eq:regret}
\mathrm{Regret}(R^{(k)}) \;=\; \frac{Z(x^{(k)\star}) \;-\; V(\Xi)}{V(\Xi)} \times 100.
\end{equation}
Since $V(\Xi) = \min_{x \in X} Z(x)$, we have $\mathrm{Regret}(R^{(k)}) \geq 0$, with equality if and only if $x^{(k)\star}$ is optimal for the full-scenario problem.
In this paper, we treat $\Xi$ as the ground-truth discrete uncertainty set, rather than as samples drawn from an underlying continuous distribution. For finite $\Xi$, Problem~\eqref{eq:2ro_formulation} can be reformulated as a deterministic-equivalent MILP whose size scales linearly with $|\Xi|$ (see Appendix~\ref{subsec:milp_reformulation}), but this formulation becomes computationally expensive when $|\Xi|$ is large.

\begin{figure}[!t]
\centering
\includegraphics[width=\textwidth]{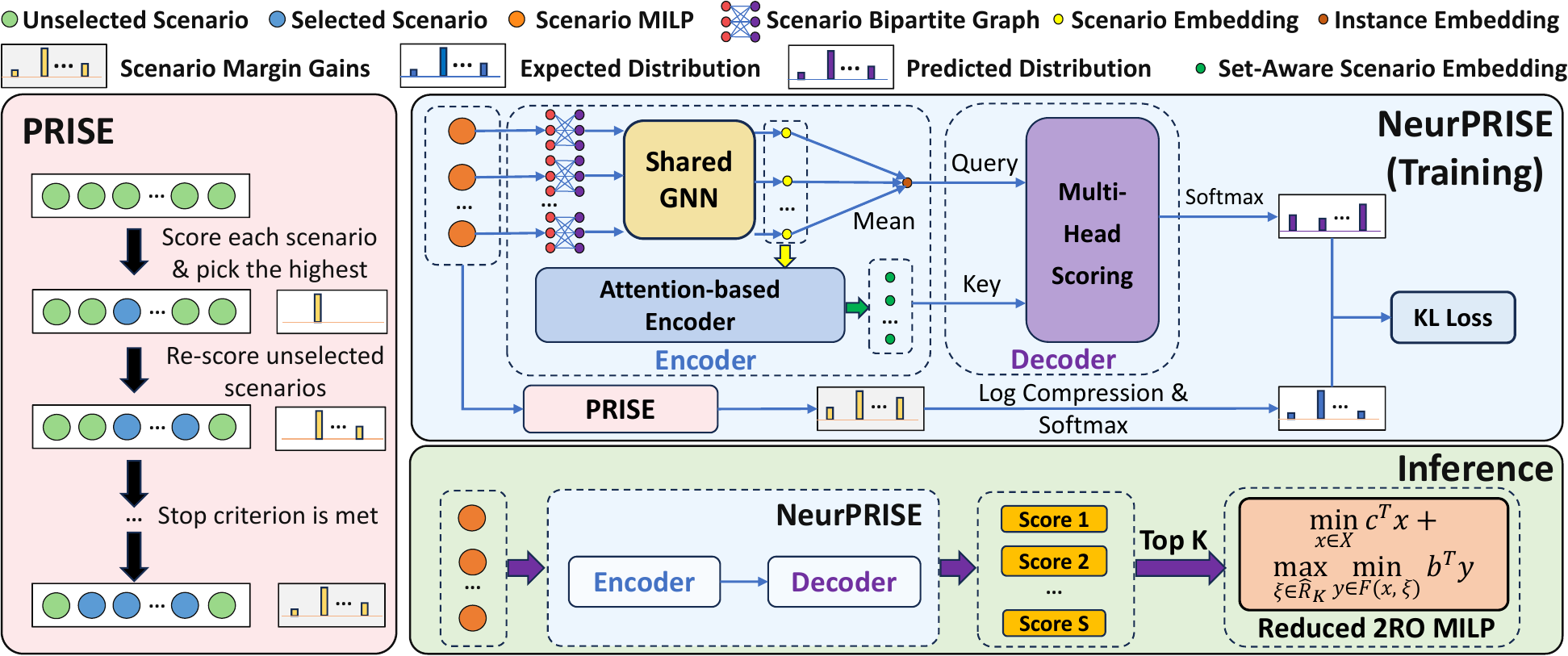}
%{fig/figure1_final.pdf}
% \vspace{0.1mm}
\caption{An illustrative overview of our approach.
\emph{Left:} PRISE; \emph{Upper Right:} NeurPRISE Framework; \emph{Lower Right:} NeurPRISE Inference.
}
\label{fig:architecture}
\end{figure}    
\section{Methodology}

In this section, we first introduce PRoblem-drIven SEquential lookahead (PRISE), a heuristic that formulates scenario reduction as an objective-aware sequential decision-making process. We then propose NeurPRISE, a neural surrogate model that alleviates computational overhead while preserving decision quality by learning from PRISE.

\subsection{PRISE}
\label{sec:heuristics}
PRISE constructs a reduced scenario set $R \subseteq \Xi$ sequentially, scoring each candidate scenario by its marginal impact on the reduced-scenario objective value $V(R)$ (Eq.~\eqref{eq:V_def}).
At step $t$, given the selected set $R_t$ (initialized as $R_0 = \emptyset$), PRISE evaluates each candidate $\xi_j \in \Xi \setminus R_t$ by solving the reduced problem on $R_t \cup \{\xi_j\}$ and selects:
\begin{equation}
\xi_t^{\mathrm{PRISE}} \in
\arg\max_{\xi_j \in \Xi \setminus R_t} V(R_t \cup \{\xi_j\}).
\label{eq:prise_scenario_choice}
\end{equation}
Equivalently, PRISE selects the scenario with the largest marginal gain $V(R_t \cup \{\xi_j\})-V(R_t)$: the scenario that most increases the reduced robust objective after the first-stage decision is re-optimized. Since evaluating $V(R_t \cup \{\xi_j\})$ entails solving the reduced MILP, the selection is inherently \emph{problem-driven}---it depends on the full problem structure (objective, constraints, recourse), not on the scenario alone. PRISE terminates after selecting a small fraction of $|\Xi|$ scenarios (Appendix~\ref{subsec:prise_empirical_analysis}).
The full procedure is detailed in Algorithm~\ref{alg:prise_supervision_fixed2}.
PRISE is characterized by two key properties:
1) \emph{Monotonicity}: $V(R)$ is monotonically non-decreasing as scenarios are added (Proposition~\ref{prop:prise_monotonicity}).
2) \emph{Empirical diminishing returns}: Although $V(R)$ is not submodular in general (Remark~\ref{rem:not_submodular}), marginal gains decrease rapidly in practice.  This diminishing-returns behavior motivates the loss design in Section~\ref{sec:loss}. 
PRISE's iterative structure is similar to the column-and-constraint generation~\cite{zhao2012exact,zeng2013solving}.
Appendix~\ref{subsec:prise_theo_guar} presents the theoretical guarantees
for PRISE, while Appendix~\ref{subsec:prise_vs_ccg} clarifies its relationship
to column-and-constraint generation.

\begin{algorithm}[t]
\caption{PRISE}
\label{alg:prise_supervision_fixed2}
\begin{algorithmic}[1]
\STATE \textbf{Input:} scenario set $\Xi{=}\{\xi_1,\dots,\xi_S\}$ with $S{=}|\Xi|$; maximum budget $K$; tolerance $\epsilon{\ge}0$
\STATE \textbf{Output:} compression budget $\hat{k}$; reduced set $R$; supervision data $\mathcal{D}$
\STATE Initialize $R \leftarrow \emptyset$, $\mathcal{D} \leftarrow \emptyset$, and $v_{\mathrm{prev}} \leftarrow 0$
\WHILE{$|R| < K$}
    \FOR{each unselected scenario $\xi_j \in \Xi \setminus R$}
        \STATE $\mathrm{score}(\xi_j) \leftarrow V(R \cup \{\xi_j\})$
    \ENDFOR
    \STATE $\xi^\star \leftarrow$ scenario with highest score
    \STATE $\Delta \leftarrow \mathrm{score}(\xi^\star) - v_{\mathrm{prev}}$
    \IF{$\Delta \le \epsilon$}
        \STATE \textbf{break}
    \ENDIF
    \STATE Add $(R,\xi^\star,\Delta)$ to supervision data $\mathcal{D}$
    \STATE Add $\xi^\star$ to the reduced set: $R \leftarrow R \cup \{\xi^\star\}$
    \STATE $v_{\mathrm{prev}} \leftarrow \mathrm{score}(\xi^\star)$
\ENDWHILE
\STATE $\hat{k} \leftarrow |R|$
\STATE \textbf{return} $\hat{k}$, $R$, $\mathcal{D}$
\end{algorithmic}
\end{algorithm}

% Computational complexity remark moved to Appendix~\ref{subsec:prise_empirical_analysis}.

\subsection{NeurPRISE}
While PRISE produces high-quality reduced sets, it requires solving $\mathcal{O}(K{\cdot}|\Xi|)$ MILPs per instance. We amortize this cost through imitation learning: NeurPRISE is a neural surrogate that learns to replicate PRISE's selection in a single forward pass. Formally, we frame scenario reduction as a set-selection task, where our goal is to learn a scoring function that approximates PRISE's marginal-gain ranking.
NeurPRISE is envisioned to capture problem-driven selection logic and efficiently predict the scenario set, substantially reducing computational overhead.
% This allows NeurPRISE to capture problem-driven selection logic while substantially reducing computational overhead.
% capturing the problem-driven selection logic while significantly reducing computational overhead.

\subsubsection{Model Architecture}
\label{sec:arch}
Given a 2RO instance with scenario set $\Xi$, the encoder represents each scenario subproblem independently and captures cross-scenario dependencies; the decoder then scores each scenario against an instance context. The model outputs logits for all $|\Xi|$ scenarios, enabling flexible top-$k$ selection without retraining. An overview is provided in Figure~\ref{fig:architecture}.
% the top $k$ scenarios. 
% Notably, this framework is not restricted to a fixed $k$ after training, providing the flexibility to generate a reduced set of any size $k \le K$. 
% Figure~\ref{fig:architecture} provides a high-level overview of the pipeline.

\textbf{Encoder.} We model each scenario $\xi_j \in \Xi$ as a bipartite graph $G_j$. We use the standard constraint--variable bipartite representation~\cite{gasse2019exact}. Details on graph construction, GNN layers
% GINE~\cite{xu2018how, hu2019strategies} layers, 
and feature engineering are given in Appendix~\ref{subsec:graph_encoding} (Figure~\ref{fig:graph_bipartite}). Each scenario graph is encoded into a $d_s$-dimensional embedding via GNN followed by global mean pooling over the node embeddings within $G_j$:
\begin{equation}
    \mathbf{h}_j = \mathrm{Pool}\left( \mathrm{GNN}_{\phi}(G_j) \right) \in \mathbb{R}^{d_s}.
\end{equation}
To capture inter-scenario dependencies, we refine these embeddings through a Transformer encoder, yielding set-aware scenario embeddings:
\begin{equation}
    \mathbf{Z}_{\mathrm{sce}} = \mathrm{TransformerEnc}(\mathbf{h}_{1:|\Xi|}) \in \mathbb{R}^{|\Xi| \times d_s}.
\end{equation}

\textbf{Decoder.} The decoder produces per-scenario importance scores by relating an instance context to individual scenarios via multi-head scoring. The instance context embedding is obtained by averaging the pre-Transformer GNN embeddings:
\begin{equation}
    \mathbf{z}_{\mathrm{inst}} = \frac{1}{|\Xi|}\sum_{j=1}^{|\Xi|} \mathbf{h}_j \;\in\; \mathbb{R}^{d_s}.
\end{equation}

The decoder receives two inputs: 1) \emph{Instance Context ($\mathbf{z}_{\mathrm{inst}}$):} A vector $[1 \times d_s]$ representing the aggregated instance-level structure, which serves as the \textit{Query} ($Q$).
2) \emph{Set-Aware Scenario Embeddings ($\mathbf{Z}_{\mathrm{sce}}$):} A matrix $[|\Xi| \times d_s]$ representing the encoded features of all $|\Xi|$ scenarios, which serves as the \textit{Key} (denoted $\mathbf{K}_{\mathrm{attn}}$ to distinguish from budget $k$).

% \paragraph{Multi-head Scoring.}
The inputs are projected into $H$ parallel heads with head dimension $d_k$. For each head $h \in \{1, \dots, H\}$ and scenario $j$, we compute a scaled dot-product score:
\begin{equation}
s_{h,j} = \frac{(\mathbf{z}_{\mathrm{inst}}\mathbf{W}_h^Q) \cdot (\mathbf{Z}_{\mathrm{sce},j}\mathbf{W}_h^K)}{\sqrt{d_k}} \in \mathbb{R},
\end{equation}
where $\mathbf{W}_h^Q, \mathbf{W}_h^K$ are learnable projection matrices. Each head captures a different aspect of scenario importance relative to the first-stage structure.
% \paragraph{Final Scoring.}
The $H$ per-head scores for each scenario are combined via a mixing multilayer perceptron (MLP) to produce the final logit:
\begin{equation}
z_j = \mathrm{MLP}\!\left([s_{1,j},\, \dots,\, s_{H,j}]\right), \quad j \in \{1, \dots, |\Xi|\},
\end{equation}
where the MLP maps $\mathbb{R}^H \to \mathbb{R}$ via a single hidden layer with ReLU activation. In parallel, it produces a vector of logits $\{z_j\}_{j=1}^{|\Xi|}$ for all scenarios, representing the predicted importance of each scenario. We refer to the complete learned pipeline---GNN-Transformer encoder, multi-head scoring decoder, and top-$k$ selection---as \textbf{NeurPRISE}. NeurPRISE scores all $|\Xi|$ candidates in a single forward pass.
\subsubsection{Loss Function}
\label{sec:loss}

PRISE assigns each scenario a one-shot importance score: for each selected scenario $\xi_{j_t}$, the marginal gain is scored by $\Delta_t = V(R_t \cup \{\xi_{j_t}\}) - V(R_t)$ (Algorithm~\ref{alg:prise_supervision_fixed2}, line~8), and unselected scenarios receive a score of zero. We collect these into a target gain vector $g \in \mathbb{R}^{|\Xi|}$ with $g_j = \Delta_t$ if scenario $j$ was selected at step $t$, and $g_j = 0$ otherwise. Due to the diminishing-returns property (Section~\ref{sec:heuristics}), the gains $g_j$ exhibit a steep decay: the most critical scenario may contribute a gain orders of magnitude larger than later-selected ones (Fig.~\ref{fig:prise_marginal_gains}). Without compression, the softmax target would concentrate nearly all mass on a single scenario. We apply log-compression to flatten this decay:
\begin{equation}
    y_j = \log(1 + g_j).
    \label{eq:log_compression}
\end{equation}

We convert these compressed gains into a soft target distribution $P$ over all $|\Xi|$ scenarios via temperature-scaled softmax, and define the predicted distribution $\hat{P}$ from its one-shot logits $z_j$:
\begin{equation}
    P_j = \frac{\exp(y_j / \tau)}{\sum_{\ell=1}^{|\Xi|} \exp(y_\ell / \tau)}, \qquad
    \hat{P}_j = \frac{\exp(z_j)}{\sum_{\ell=1}^{|\Xi|} \exp(z_\ell)},
    \label{eq:target_pred_dist}
\end{equation}
where $\tau$ is a temperature that flattens the target, ensuring unselected scenarios receive a non-trivial gradient signal. The training objective is:
\begin{equation}
    \mathcal{L}(\theta) = D_{\mathrm{KL}}\!\left( P \,\|\, \hat{P} \right).
    \label{eq:kl_loss}
\end{equation}
NeurPRISE produces $\hat{P}$ in a single forward pass over all $|\Xi|$ scenarios. At inference, the user selects the top-$k$ scenarios for any desired budget $k \le K$, with no step iteration required.

We refer to this objective as the \emph{Score-Weighted KL} (SWKL) loss. The forward direction $D_{\mathrm{KL}}(P \| \hat{P})$ weights the loss by the target $P$, ensuring every high-gain scenario receives gradient signal rather than only the single best one. SWKL performs consistently across all evaluated problems and instance sizes; a full ablation against BCE and R2R is in Appendix~\ref{subsubsec:ablation_loss_comparison}.

\section{Experiments}
\label{sec:experiments}

All experiments are conducted on a server with an AMD EPYC 7543 32-Core CPU (2.80\,GHz) and NVIDIA A6000 Ada GPUs with 48\,GB of VRAM. Optimization problems are solved using Gurobi 12.0.0 \cite{gurobi}. The neural architecture is implemented using PyTorch 2.8.0 \cite{paszke2019pytorch} and PyG 2.7.0 \cite{fey2019fastgraphrepresentationlearning}. All Gurobi calls use 40 threads and MIP gap tolerance $10^{-4}$ unless otherwise noted. Our implementation is publicly available at \url{https://github.com/lintianjue/NeurPRISE}.

\textbf{Problems.}
Since \cite{goerigk2023optimal} is the most relevant prior study, we evaluate on its 2RO problem classes: \textbf{Selection (SEL)}, which selects a minimum-cost subset of items with uncertain costs; \textbf{Vertex Cover (VC)}, which finds a minimum-cost node cover with uncertain node costs. Detailed formulations are in Appendix~\ref{subsec:problem_formulations}. Instances are denoted SEL-$n$-$s$ and VC-$n$-$s$, where $n$ is the number of items/nodes and $s$ the number of scenarios. Moreover, we also evaluate the \textbf{Capacitated Facility Location Problem (CFLP)} where uncertainty appears in the constraints. Given space limitations, we defer its formulation, experimental setup, and results to Appendix~\ref{sec:cflp_appendix}.

\textbf{Instance Generation.}
\label{par:instance_gen}
We generate 2{,}500 \emph{small} instances (SEL-20-50, VC-20-50), 500 for \emph{medium} instances (SEL-50-50, VC-50-50), and 500 for \emph{large} instances (SEL-50-100, VC-50-100). Per-node scenario costs are drawn i.i.d.\ from $\mathrm{Uniform}[1,100]$, following \cite{goerigk2023optimal}. 
%The smaller count for medium and large reflects the super-linear growth of PRISE labeling cost with problem size; the generalization results in Section~\ref{sec:gen_size} show that training on small-scale labels transfers effectively, reducing the need for costly large-scale labeling.

\textbf{Training.}
For each instance size, we split instances into training, validation, and test sets (80\%/10\%/10\%), yielding 250 test instances for small and 50 each for medium and large. We train with AdamW for at most 200 epochs with early stopping (patience 10). The evaluation pipeline is fully deterministic ($\operatorname{top\text{-}k}$ scoring). 
The validation set is used for model selection; we report the checkpoint results with the lowest validation loss. In the main paper, we report results with seed~42. We further examine variance across five training seeds in Section~\ref{subsec:seed_variance}.
%Tables report single-seed (seed~42) results for readability. A five-seed variance analysis (Appendix~\ref{subsec:seed_variance}, Table~\ref{tab:seed_variance}) confirms that training stochasticity has negligible effect: max per-cell standard deviation is far smaller than the performance gaps between methods. 
%Random baseline entries are five-sampling-seed averages (per-cell standard deviation ${\leq}0.9$\,pp). Full hyperparameters are in Appendix~\ref{subsec:training_details}.

\textbf{Baselines.}
We compare against three categories of methods.
\emph{Reference methods:}
\textbf{Exact} (Gurobi, full scenario set) provides the exact full-scenario objective value $V(\Xi)$;
\textbf{PRISE} (our expert heuristic) provides the labeling oracle for NeurPRISE training.
\emph{Scenario reduction baselines:}
\textbf{Random} ($k$ scenarios sampled uniformly);
\textbf{K-means}~\cite{goerigk2023optimal} (cluster $S$ scenarios into $k$ groups, select the nearest original scenario to each centroid);
\textbf{MaxSum}~\cite{goerigk2024data}, which ranks scenarios by the sum of their cost coefficients and retains the top-$k$;
\textbf{SOR}~\cite{goerigk2023optimal}, which optimally selects $k$ representative scenarios with a provable worst-case approximation guarantee.
\emph{Learning-based 2RO:}
\textbf{GK-RF}~\citep{goerigk2024data}, a random-forest classifier that relies on comprehensive feature engineering and is trained on scenario-relevance labels from CCG; it retains the top-$k$ scenarios by predicted relevance;
\textbf{Neur2RO}~\cite{dumouchelle2023neur2ro}, an end-to-end ML approach that trains a neural surrogate of the two-stage objective and embeds it into a CCG loop. 
% \textcolor{red}{We follow the original codebase and adapt it to our discrete setting.}
%\emph{Recourse diversification:}
%\textbf{K-adaptability}~\cite{subramanyam2020k}, which pre-computes $K$ recourse policies and assigns each scenario to its best-fitting policy.

% File: compare_table.tex
% Auto-generated by scripts/gen_table1_regret.py

\begin{table*}[!t]
\centering
\caption{Performance comparison on small-scale instances (250 test instances). A dash~(-) denotes a $k$-agnostic method with identical results across budgets.}
\label{tab:full_results_small}
\begin{small}
\setlength{\tabcolsep}{2pt}
\begin{tabular}{@{}ll c@{\hspace{2pt}}c c@{\hspace{2pt}}c c@{\hspace{2pt}}c c@{\hspace{2pt}}c@{}}
\toprule[1.2pt]
\multirow{2}{*}{\textsc{Prob.}} & \multirow{2}{*}{\textsc{Method}} & \multicolumn{2}{c}{$k=1$} & \multicolumn{2}{c}{$k=2$} & \multicolumn{2}{c}{$k=4$} & \multicolumn{2}{c}{$k=6$} \\
\cmidrule(lr){3-4} \cmidrule(lr){5-6} \cmidrule(lr){7-8} \cmidrule(lr){9-10}
& & \textsc{Reg.}$\downarrow$ & \textsc{Time(s)} & \textsc{Reg.}$\downarrow$ & \textsc{Time(s)} & \textsc{Reg.}$\downarrow$ & \textsc{Time(s)} & \textsc{Reg.}$\downarrow$ & \textsc{Time(s)} \\
\midrule[1.2pt]
 \multirow{8}{*}{\rotatebox[origin=c]{90}{SEL}} & Exact & 0.00 & 6.1 & - & - & - & - & - & - \\
  & PRISE & 5.93 & 49.8 & 2.14 & 91.7 & 0.92 & 126.0 & 0.82 & 131.7 \\
\cmidrule{2-10}
  & Random & 22.11 & 3.8 & 16.13 & 3.8 & 12.29 & 3.9 & 9.72 & 4.1 \\
  & SOR & 5.45 & 25.9 & 4.57 & 52.9 & 3.44 & 63.4 & 2.83 & 64.5 \\
  & K-means & 10.89 & 25.8 & 9.02 & 60.8 & 7.35 & 90.5 & 7.44 & 110.1 \\
  & MaxSum & 6.93 & 4.7 & 4.50 & 3.8 & 2.47 & 3.9 & 1.96 & 4.2 \\
\cmidrule{2-10}
  & GK-RF & \textbf{4.00} & 26.6 & 2.73 & 27.8 & 1.74 & 28.2 & 1.16 & 30.1 \\
  & Neur2RO & 5.88 & 2.2 & - & \textbf{-} & - & \textbf{-} & - & \textbf{-} \\
  & NeurPRISE & 5.59 & \textbf{2.0} & \textbf{2.18} & 2.3 & \textbf{0.79} & 2.8 & \textbf{0.42} & 3.4 \\
\midrule
 \multirow{8}{*}{\rotatebox[origin=c]{90}{VC}} & Exact & 0.00 & 450.8 & - & - & - & - & - & - \\
  & PRISE & 13.85 & 82.1 & 7.91 & 583.6 & 3.00 & 3{,}008.5 & 1.04 & 5{,}680.2 \\
\cmidrule{2-10}
  & Random & 16.11 & 3.8 & 12.49 & 13.0 & 10.58 & 26.9 & 9.65 & 40.6 \\
  & SOR & 11.15 & 27.4 & 8.52 & 67.9 & 6.71 & 92.0 & 5.63 & 111.6 \\
  & K-means & 12.13 & 25.6 & 10.88 & 72.4 & 9.33 & 114.5 & 7.98 & 149.1 \\
  & MaxSum & 11.75 & 4.5 & \textbf{8.25} & 14.2 & 5.76 & 31.2 & 4.40 & 46.9 \\
\cmidrule{2-10}
  & GK-RF & 14.43 & 126.3 & 9.45 & 156.0 & 5.00 & 199.5 & 3.29 & 244.3 \\
  & Neur2RO & \textbf{11.13} & 5.8 & - & \textbf{-} & - & \textbf{-} & - & \textbf{-} \\
  & NeurPRISE & 14.75 & \textbf{2.9} & 9.36 & 14.8 & \textbf{4.20} & 38.5 & \textbf{2.34} & 57.8 \\
\bottomrule[1.2pt]
\end{tabular}
\end{small}
\end{table*}

% File: compare_table_medium_large.tex
% Merged medium + large comparison table.

\begin{table*}[!t]
\centering
\caption{Medium- and large-scale comparison (50 test instances each). All notation follows Table~\ref{tab:full_results_small}.}
\label{tab:full_results_medium}
\phantomsection\label{tab:full_results_large}
\begin{small}
\setlength{\tabcolsep}{2pt}
\begin{tabular}{@{}lll c@{\hspace{2pt}}c c@{\hspace{2pt}}c c@{\hspace{2pt}}c c@{\hspace{2pt}}c@{}}
\toprule[1.2pt]
\multirow{2}{*}{\textsc{Prob.}} & \multirow{2}{*}{\textsc{Scale}} & \multirow{2}{*}{\textsc{Method}} & \multicolumn{2}{c}{$k=1$} & \multicolumn{2}{c}{$k=2$} & \multicolumn{2}{c}{$k=4$} & \multicolumn{2}{c}{$k=6$} \\
\cmidrule(lr){4-5} \cmidrule(lr){6-7} \cmidrule(lr){8-9} \cmidrule(lr){10-11}
& & & \textsc{Reg.}$\downarrow$ & \textsc{Time(s)} & \textsc{Reg.}$\downarrow$ & \textsc{Time(s)} & \textsc{Reg.}$\downarrow$ & \textsc{Time(s)} & \textsc{Reg.}$\downarrow$ & \textsc{Time(s)} \\
\midrule[1.2pt]
 \multirow{14}{*}{\rotatebox[origin=c]{90}{SEL}}
  & \multirow{7}{*}{Med.}
  & Exact & 0.00 & 4.8 & - & - & - & - & - & - \\
  & & PRISE & 6.64 & 12.0 & 2.11 & 24.3 & 0.90 & 41.2 & 0.59 & 45.3 \\
\cmidrule{3-11}
  & & Random & 13.56 & 0.7 & 9.74 & 0.8 & 8.98 & 1.2 & 7.77 & 1.5 \\
  & & SOR & 8.01 & 5.3 & 6.00 & 16.6 & 5.37 & 21.8 & 4.91 & 20.2 \\
  & & MaxSum & 6.66 & 1.5 & 6.27 & 0.8 & 4.24 & 1.2 & 3.33 & 1.5 \\
\cmidrule{3-11}
  & & GK-RF & 4.62 & 9.2 & \textbf{3.91} & 9.6 & 2.36 & 9.8 & 2.09 & 10.2 \\
  & & Neur2RO & 5.56 & \textbf{0.6} & - & \textbf{-} & - & \textbf{-} & - & \textbf{-} \\
  & & NeurPRISE & \textbf{4.38} & 0.7 & 4.12 & 0.8 & \textbf{2.19} & 1.2 & \textbf{1.76} & 1.5 \\
\cmidrule{2-11}
  & \multirow{7}{*}{Large}
  & Exact & 0.00 & 9.6 & - & - & - & - & - & - \\
  & & PRISE & 5.80 & 21.9 & 2.55 & 47.7 & 0.87 & 95.2 & 0.49 & 112.9 \\
\cmidrule{3-11}
  & & Random & 13.54 & 0.8 & 10.51 & 0.8 & 8.44 & 1.2 & 8.88 & 1.4 \\
  & & SOR & 8.30 & 20.3 & 6.32 & 107.9 & 6.05 & 144.3 & 5.00 & 169.2 \\
  & & MaxSum & 7.33 & 0.8 & 5.60 & 0.9 & 4.03 & 1.3 & 3.26 & 1.5 \\
\cmidrule{3-11}
  & & GK-RF & 5.00 & 13.1 & 4.03 & 13.5 & \textbf{2.53} & 13.8 & 1.75 & 14.2 \\
  & & Neur2RO & 6.58 & \textbf{0.7} & - & \textbf{-} & - & \textbf{-} & - & \textbf{-} \\
  & & NeurPRISE & \textbf{4.18} & 0.8 & \textbf{3.61} & 1.0 & \textbf{2.53} & 1.3 & \textbf{1.62} & 1.8 \\
\midrule
 \multirow{14}{*}{\rotatebox[origin=c]{90}{VC}}
  & \multirow{7}{*}{Med.}
  & Exact & 0.00 & 1{,}690.5 & - & - & - & - & - & - \\
  & & PRISE & 16.35 & 33.5 & 9.17 & 526.5 & 3.83 & 2{,}965.7 & 2.28 & 7{,}594.9 \\
\cmidrule{3-11}
  & & Random & 17.10 & 2.1 & 12.51 & 10.5 & 10.15 & 26.9 & 9.39 & 48.1 \\
  & & SOR & 13.77 & 5.9 & 11.39 & 29.5 & 9.06 & 49.6 & 8.40 & 75.7 \\
  & & MaxSum & 13.65 & \textbf{1.3} & \textbf{9.56} & 11.3 & 6.90 & 31.9 & 5.20 & 71.4 \\
\cmidrule{3-11}
  & & GK-RF & 16.57 & 79.3 & 9.79 & 94.1 & \textbf{6.16} & 131.8 & 4.50 & 203.7 \\
  & & Neur2RO & \textbf{12.87} & 3.1 & - & \textbf{-} & - & \textbf{-} & - & \textbf{-} \\
  & & NeurPRISE & 15.26 & \textbf{1.3} & 9.69 & 12.6 & 6.50 & 40.1 & \textbf{3.65} & 100.5 \\
\cmidrule{2-11}
  & \multirow{7}{*}{Large}
  & Exact & 0.00 & 4{,}781.7 & - & - & - & - & - & - \\
  & & PRISE & 16.07 & 63.6 & 9.20 & 1{,}016.5 & 3.98 & 6{,}166.4 & 2.88 & 18{,}713.4 \\
\cmidrule{3-11}
  & & Random & 15.61 & \textbf{1.3} & 12.71 & 9.3 & 9.47 & 26.9 & 8.69 & 49.5 \\
  & & SOR & 13.09 & 19.6 & 9.52 & 112.3 & 7.65 & 165.6 & 6.65 & 218.4 \\
  & & MaxSum & 13.41 & \textbf{1.3} & \textbf{7.99} & 9.8 & 6.91 & 28.1 & 5.37 & 55.6 \\
\cmidrule{3-11}
  & & GK-RF & 18.36 & 104.2 & 10.26 & 118.5 & 6.91 & 148.8 & \textbf{4.82} & 207.8 \\
  & & Neur2RO & \textbf{12.85} & 2.3 & - & \textbf{-} & - & \textbf{-} & - & \textbf{-} \\
  & & NeurPRISE & 14.92 & 1.5 & 9.06 & 11.9 & \textbf{6.85} & 36.5 & 4.89 & 87.1 \\
\bottomrule[1.2pt]
\end{tabular}
\end{small}
\end{table*}

\subsection{Comparison Analysis}

We evaluate methods across reduction budgets $k \in \{1,2,4,6\}$, where $k$ denotes the number of scenarios retained, and report mean regret and runtime summed over the test instances. For scenario reduction methods, runtime includes only scenario selection and solving the reduced MILP; loading instances and models is excluded so that it does not affect the comparison. A dash~(-) denotes a $k$-agnostic method with identical results across scenario-reduction budgets. Bold marks the lowest regret and runtime among non-oracle methods in each column (an exact tie is bolded for both). Tables~\ref{tab:full_results_small} and~\ref{tab:full_results_medium} present results on small, medium and large instances, respectively, where consistent trends hold across all three scales. PRISE attains the lowest regret in most columns with $k{\ge}2$, at a high wall-clock cost, as expected for a labeling oracle. NeurPRISE amortizes this cost by learning from PRISE labels, replacing repeated lookahead MILP evaluations with a single scoring pass. Overall, NeurPRISE achieves strong performance relative to non-oracle baselines: it attains the lowest regret in 14 of the 24 (problem, scale, $k$) columns. On SEL, NeurPRISE achieves the lowest regret in most of the columns. On VC, NeurPRISE is competitive with the leading baselines at $k{=}2$ and is the best, or within a small margin of the best, among all scenario-reduction and learning-based methods when $k{\ge}4$. NeurPRISE exhibits runtime comparable to scenario-based methods, while being substantially faster than PRISE, Exact, and GK-RF, whose features require solving a MILP for each scenario online, making it $2$--$69\times$ slower than NeurPRISE on VC at medium and large scale.

\subsection{Flexibility and Scalability}
\label{subsec:flexibility}

\textbf{Flexible scenario budgets.}
NeurPRISE produces a full scenario ranking, from which reduced sets of any size
$k$ can be extracted without retraining. This allows practitioners to flexibly adjust
the scenario budget at deployment time without regenerating labels or
re-running the training pipeline. Given the specified scenario budget $k$, 
the downstream MILP runtime can also be traded by relaxing solver gap tolerances.
When $k=4$ on large VC instances, a $7\%$ gap tolerance yields a $2\times$
speedup while regret remains nearly unchanged. With a more aggressive $25\%$
tolerance, runtime improves by $7\times$ at the cost of only a $3.4\%$ increase
in regret, providing a practical quality--speed knob when tight runtime constraints
are present (Appendix~\ref{sec:solver_tolerance}).

\textbf{Sub-linear runtime scaling with scenario count.}
\label{subsubsec:sub-linear-runtime}
NeurPRISE substantially limits the downstream optimization cost with much fewer scenarios, saving considerable time relative to Exact method. As downstream MILP is solved only on the selected $k$
scenarios, its complexity is governed by $k$, not the original scenario count
$S$. An ablation on the same instances
(Appendix~\ref{subsec:runtime_scaling}) confirms the speedup effect: on large instances, increasing $S$
by $5\times$ raises total runtime by only $1.2\times$. This sub-linear scaling
allows NeurPRISE to remain practical when many candidate scenarios are available
at deployment.

\subsection{Generalization}
\label{sec:gen_size}
Since PRISE labeling cost grows super-linearly with problem size, generalization without retraining is essential. 
NeurPRISE is particularly suitable for this setting because its learned scoring rule is amortized across instances: once trained, the same model can be reused to rank scenarios under different scenario budgets, larger scenario sets, and shifted deployment distributions without generating new PRISE labels. We evaluate NeurPRISE along three axes:

\textbf{Problem size (up to $5\times$).}
\label{sec:gen_problem_size}
NeurPRISE generalizes to problem sizes up to $5\times$ larger than its training scale without retraining, a critical property since PRISE labeling cost at target scale can exceed the training cost. The large test instances scale problem size and scenario count relative to training; the small-scale model still matches or outperforms target-scale models trained directly on those sizes (Table~\ref{tab:generalization_merged}, Appendix~\ref{subsubsec:generalization_transfer}). On ultra-large instances at $5\times$ problem size (Table~\ref{tab:ultra_generalization}), NeurPRISE achieves the lowest regret on SEL across all budgets and on VC at $k{\ge}4$.

\textbf{Scenario count (up to $4\times$).}
% \textcolor{red}{While Section~\ref{subsubsec:sub-linear-runtime} showed runtime grows only sub-linearly with the scenario count, 
We demonstrate that the superior regret performance of our method effectively generalizes across substantially different numbers of scenarios. NeurPRISE scales naturally to larger scenario sets because the attention-based encoder processes scenario embeddings independently. Models trained on $s{=}50$ transfer to $s{=}100, 200$ without retraining (Table~\ref{tab:scenario_generalization}). NeurPRISE maintains its advantage across scenario counts.

% File: generalization_sidebyside.tex
% Tables 2 (ultra 5x) and 3 (scenario generalization) side-by-side.

\noindent
\begin{minipage}[t]{0.46\textwidth}
\centering\small
\setlength{\tabcolsep}{3pt}
\captionof{table}{Generalization to $5\times$ problem size.}
\label{tab:ultra_generalization}
\begin{tabular}{@{}llcccc@{}}
\toprule[1.2pt]
\multirow{2}{*}{\textsc{Prob.}} & \multirow{2}{*}{\textsc{Method}}
  & \multicolumn{4}{c}{\textsc{Regret (\%) $\downarrow$}} \\
\cmidrule(lr){3-6}
& & $k{=}1$ & $k{=}2$ & $k{=}4$ & $k{=}6$ \\
\midrule[1.2pt]
\multirow{3}{*}{\rotatebox[origin=c]{90}{SEL}}
 & MaxSum & 6.45 & 5.66 & 3.47 & 2.96 \\
 & Neur2RO & 6.85 & - & - & - \\
 & NeurPRISE & \textbf{5.21} & \textbf{3.26} & \textbf{2.12} & \textbf{0.81} \\
\midrule
\multirow{3}{*}{\rotatebox[origin=c]{90}{VC}}
 & MaxSum & 14.88 & \textbf{9.75} & 7.06 & 5.76 \\
 & Neur2RO & \textbf{14.79} & - & - & - \\
 & NeurPRISE & 15.98 & 10.09 & \textbf{5.89} & \textbf{4.20} \\
\bottomrule[1.2pt]
\end{tabular}
\end{minipage}
\hfill
\begin{minipage}[t]{0.52\textwidth}
\centering\small
\setlength{\tabcolsep}{2pt}
\captionof{table}{Scenario generalization: $S{=}50$ to $100, 200$.}
\label{tab:scenario_generalization}
\begin{tabular}{@{}ll ccc ccc@{}}
\toprule[1.2pt]
\textsc{Pr.} & \textsc{Method}
  & \multicolumn{3}{c}{$S{=}100$ $\downarrow$}
  & \multicolumn{3}{c@{}}{$S{=}200$ $\downarrow$} \\
\cmidrule(lr){3-5} \cmidrule(l){6-8}
& & $k{=}2$ & $k{=}4$ & $k{=}8$
  & $k{=}4$ & $k{=}8$ & $k{=}16$ \\
\midrule[1.2pt]
\multirow{3}{*}{\rotatebox[origin=c]{90}{SEL}}
 & MaxSum  & 6.00 & 3.57 & 2.44 & 2.84 & 1.99 & 1.23 \\
 & Neur2RO & 6.35 & - & - & 5.09 & - & - \\
 & NeurPRISE & \textbf{3.35} & \textbf{1.25} & \textbf{0.32} & \textbf{1.24} & \textbf{0.55} & \textbf{0.07} \\
\midrule
\multirow{3}{*}{\rotatebox[origin=c]{90}{VC}}
 & MaxSum  & 9.49 & 5.55 & 3.44 & 5.59 & 3.56 & 1.88 \\
 & Neur2RO & 12.29 & - & - & 10.52 & - & - \\
 & NeurPRISE & \textbf{8.96} & \textbf{5.05} & \textbf{1.75} & \textbf{4.78} & \textbf{2.48} & \textbf{0.87} \\
\bottomrule[1.2pt]
\end{tabular}
\end{minipage}

\textbf{Distribution shift.}
\label{sec:gen_distribution}
We test whether NeurPRISE, trained on uniform scenarios, generalizes to Normal and Multimodal (MM) distributions without retraining (Tables~\ref{tab:distribution_generalization}--\ref{tab:distribution_mm}; setup details in Appendix~\ref{subsec:appendix_distribution}). NeurPRISE$_\text{uni}$ outperforms MaxSum and Neur2RO across most settings under both shifted distributions, and performs comparably to in-distribution models (NeurPRISE$_\text{norm}$, NeurPRISE$_\text{mm}$), confirming that a single uniform-trained model generalizes across distribution families.

% File: distribution_table.tex
% Updated 2026-04-28: Split into two independent side-by-side tables.

\noindent
\begin{minipage}[t]{0.48\textwidth}
\centering\small
\setlength{\tabcolsep}{2pt}
\captionof{table}{Distribution shift (Normal).}
\label{tab:distribution_generalization}
\begin{tabular}{@{}llcccc@{}}
\toprule[1.2pt]
\textsc{Prob.} & \textsc{Method} & $k{=}1$ $\downarrow$ & $k{=}2$ $\downarrow$ & $k{=}4$ $\downarrow$ & $k{=}6$ $\downarrow$ \\
\midrule[1.2pt]
\multirow{4}{*}{\rotatebox[origin=c]{90}{SEL}}
 & MaxSum          & 3.09           & 2.91           & 2.13           & 1.62 \\
 & Neur2RO         & 2.96           & -              & -              & - \\
 & NeurPRISE$_\text{uni}$   & \textbf{2.27}  & 2.18           & 1.89           & 1.56 \\
 & NeurPRISE$_\text{norm}$  & \textbf{2.27}  & \textbf{2.10}  & \textbf{1.70}  & \textbf{1.35} \\
\midrule
\multirow{4}{*}{\rotatebox[origin=c]{90}{VC}}
 & MaxSum          & 4.69           & 3.93           & 2.35           & 0.70 \\
 & Neur2RO         & 4.69           & -              & -              & - \\
 & NeurPRISE$_\text{uni}$   & \textbf{4.56}  & \textbf{3.71}  & 2.39           & 0.89 \\
 & NeurPRISE$_\text{norm}$  & 4.67           & 4.07           & \textbf{2.23}  & \textbf{0.62} \\
\bottomrule[1.2pt]
\end{tabular}
\end{minipage}
\hfill
\begin{minipage}[t]{0.48\textwidth}
\centering\small
\setlength{\tabcolsep}{2pt}
\captionof{table}{Distribution shift (MM).}
\label{tab:distribution_mm}
\begin{tabular}{@{}llcccc@{}}
\toprule[1.2pt]
\textsc{Prob.} & \textsc{Method} & $k{=}1$ $\downarrow$ & $k{=}2$ $\downarrow$ & $k{=}4$ $\downarrow$ & $k{=}6$ $\downarrow$ \\
\midrule[1.2pt]
\multirow{4}{*}{\rotatebox[origin=c]{90}{SEL}}
 & MaxSum          & 1.91           & 1.50           & 1.08           & 0.77 \\
 & Neur2RO         & 1.55           & -              & -              & - \\
 & NeurPRISE$_\text{uni}$   & 1.19           & \textbf{0.82}  & \textbf{0.46}  & \textbf{0.28} \\
 & NeurPRISE$_\text{mm}$    & \textbf{1.17}  & 0.94           & 0.58           & 0.36 \\
\midrule
\multirow{4}{*}{\rotatebox[origin=c]{90}{VC}}
 & MaxSum          & 7.53           & \textbf{6.22}  & 5.19           & 4.18 \\
 & Neur2RO         & \textbf{7.08}  & -              & -              & - \\
 & NeurPRISE$_\text{uni}$   & 7.78           & 6.71           & \textbf{4.66}  & 3.84 \\
 & NeurPRISE$_\text{mm}$    & 7.62           & 6.56           & 4.92           & \textbf{3.82} \\
\bottomrule[1.2pt]
\end{tabular}
\end{minipage}

\section{Conclusion and Limitations}
\label{sec:limitations}
We propose PRISE, a problem-driven sequential lookahead heuristic for scenario reduction in two-stage robust optimization (2RO), and NeurPRISE, a GNN--Transformer trained to imitate PRISE using a gain-aware KL objective. Across three problem classes, NeurPRISE achieves competitive regret and runs substantially faster than PRISE. It also generalizes zero-shot to larger instances, larger scenario pools, and distribution shifts.

We acknowledge certain limitations of our approach, such as the cost of generating PRISE labels (Appendix~\ref{subsec:prise_empirical_analysis}) and the downstream MILP cost inherent to 2RO with discrete uncertainty, which grows with the size of the reduced scenario set (Appendix~\ref{subsec:runtime_scaling}). We also propose ways to alleviate both: cross-size generalization (Section~\ref{sec:gen_size}), cross-problem transfer (Appendix~\ref{subsubsec:cross_problem_transfer}), and cheaper labels from approximate PRISE solves (Appendix~\ref{subsec:cheap_labels}) reduce the offline labeling expense, while relaxed solver tolerances trade a small regret increase for faster solves (Appendix~\ref{sec:solver_tolerance}). Future work will continue to investigate cheaper labeling oracles and further reductions of the downstream optimization cost.

\renewcommand{\acksection}{\section*{Acknowledgments}}
\begin{ack}
This research is supported by the National Research Foundation, Singapore under its AI Singapore AI Research Fundamental Research Collaborative (US-NSF Researcher Call) (AISG Award No: AISG3-RP-2025-036-USNSF) and AI Singapore Programme (AISG Award No: AISG3-RP-2022-031). This work is also supported by the National Natural Science Foundation of China (Grant No.~52472461) from the Second Research Institute of the Civil Aviation Administration of China. Any opinions, findings and conclusions or recommendations expressed in this material are those of the author(s) and do not reflect the views of National Research Foundation, Singapore, or the other funding agencies. We would like to thank the anonymous reviewers and (S)ACs of NeurIPS 2026 for their constructive comments and service to the community.
\end{ack}

\clearpage
{
\small
\bibliographystyle{plainnat}
\bibliography{main}          % This tells LaTeX to look at main.bib
}

\clearpage
\appendix

% Clickable appendix-only table of contents
\startcontents[appendix]
\printcontents[appendix]{}{0}{\section*{Appendix Contents}\setcounter{tocdepth}{3}}
\clearpage

% Appendix hub file — each section in its own file for easier editing.
% To add a new appendix section, create sec/A{N}_{name}.tex and add \input here.

% A. Notation
\section{Notation}
\label{sec:notation}

For quick reference, we summarize the main symbols used throughout the paper. Core notation ($Q$, $V$, Regret) is defined in Section~\ref{sec:preliminaries}; the explicit MILP reformulations are in Appendix~\ref{subsec:milp_reformulation}.

\begin{center}
\small
\renewcommand{\arraystretch}{1.15}
\begin{tabular}{@{}l l@{}}
\toprule[1.2pt]
\textbf{Symbol} & \textbf{Meaning} \\
\midrule[1.2pt]
$x \in X$                     & first-stage decision (here-and-now) \\
$y \in F(x,\xi)$               & second-stage decision (recourse) \\
$\xi \in \Xi$                  & uncertainty realization \\
$\Xi = \{\xi_1,\dots,\xi_S\}$  & discrete uncertainty set, $|\Xi|=S$ \\
$R \subseteq \Xi$              & reduced scenario set \\
$c,\ b_\xi$                    & first- / second-stage cost vectors \\
$Q(x,\xi)$                     & recourse value, $\min_{y \in F(x,\xi)} b_\xi^\top y$ (Section~\ref{sec:preliminaries}) \\
$V(R)$                         & reduced-scenario objective value, $\min_{x\in X}\bigl[c^\top x + \max_{\xi\in R} Q(x,\xi)\bigr]$\ \ (Eq.~\ref{eq:V_def}) \\
$V(\Xi)$                       & full-scenario objective value (ground truth) \\
$Z(x)$                & full-scenario cost, $c^\top x + \max_{\xi\in\Xi} Q(x,\xi)$;\ \ $V(\Xi) = \min_{x\in X} Z(x)$ \\
$x^{(k)\star}$        & optimal first-stage solution of the reduced problem \\
$\mathrm{Regret}(R^{(k)})$ 
& $\left(Z(x^{(k)\star}) - V(\Xi)\right)/V(\Xi) \times 100$ \ \ (Eq.~\ref{eq:regret}) \\
\midrule[1.2pt]
$K$                            & user-specified scenario budget (max $|R|$) \\
$\hat{k}$                      & compression budget, i.e.\ smallest $k$ with $(V(\Xi){-}V(R^{(k)}))/V(\Xi)\le 1\%$ (Table~\ref{tab:effective_k_stats}) \\
$R_t$                & PRISE selected set at step $t$ \\
$\Delta_t$                     & PRISE marginal gain at step $t$ \\
$\mathcal{D}$                  & PRISE supervision dataset \\
\bottomrule[1.2pt]
\end{tabular}
\end{center}

% B. PRISE: Method Details (Algorithm, Monotonicity, Compression Budget)
\section{PRISE: Method Details}
\label{sec:prise_details}

\subsection{Monotonicity of the reduced-scenario objective value}
\label{subsec:prise_monotonicity}

We formalize the monotonic behavior of the reduced-scenario objective value $V(R)$, defined in Eq.~\eqref{eq:V_def}, with respect to the growth of the selected scenario set $R$.

\begin{proposition}[Monotonicity under set inclusion]
\label{prop:prise_monotonicity}
For any $R\subseteq R' \subseteq \Xi$, the objective values satisfy
\begin{equation}
V(R)\;\le\; V(R')\;\le\; V(\Xi).
\label{eq:vR_monotone}
\end{equation}
Consequently, if $\{R^{(k)}\}_{k\ge1}$ is a nested sequence with $R^{(k)}\subseteq R^{(k+1)}\subseteq \Xi$ (e.g., the PRISE construction),
then $\{V(R^{(k)})\}_{k\ge1}$ is non-decreasing, and the underestimation gap $V(\Xi)-V(R^{(k)})$ is non-increasing in $k$.
\end{proposition}

\begin{proof}
Fix any $x\in X$. Since $R\subseteq R'$, the pointwise monotonicity of the maximum gives
\begin{equation}
\max_{\xi\in R} Q(x,\xi)\;\le\;\max_{\xi\in R'} Q(x,\xi).
\label{eq:max_monotone}
\end{equation}
Adding $c^\top x$ to both sides and taking the minimum over $x \in X$ preserves the inequality, so
\begin{equation}
V(R)\;=\;\min_{x\in X}\!\left[c^\top x + \max_{\xi\in R} Q(x,\xi)\right]
\;\le\;\min_{x\in X}\!\left[c^\top x + \max_{\xi\in R'} Q(x,\xi)\right]
\;=\;V(R').
\label{eq:min_preserves}
\end{equation}
Applying the same argument with $R'\subseteq \Xi$ gives $V(R')\le V(\Xi)$, proving~\eqref{eq:vR_monotone}.
For a nested sequence $\{R^{(k)}\}$, $R^{(k)}\subseteq R^{(k+1)}$ implies $V(R^{(k)})\le V(R^{(k+1)})$ by the first inequality in~\eqref{eq:vR_monotone}.
Finally, since $V(\Xi)$ is constant, $V(\Xi)-V(R^{(k)})$ is non-increasing.
\end{proof}

\begin{remark}[Non-submodularity of $V$]
\label{rem:not_submodular}
Although PRISE empirically shows diminishing returns (Section~\ref{sec:heuristics}), the set function $V(R)$ in Eq.~\eqref{eq:V_def} is \emph{not} submodular in general. Recall that a set function $g$ is submodular if $g(S\cup\{e\})-g(S) \ge g(T\cup\{e\})-g(T)$ for all $S\subseteq T$ and $e\notin T$; informally, adding an element to a smaller set yields at least as much marginal gain as adding it to a larger set.

\paragraph{Counterexample.}
Consider a 2RO instance with three feasible first-stage decisions $X=\{a,b,c\}$ and three scenarios $\{s_1,s_2,s_3\}$. Let $f(x,\xi):=c^\top x + Q(x,\xi)$ denote the total cost (first-stage plus recourse) of decision $x$ under scenario $\xi$:
\begin{center}
\small
\begin{tabular}{@{}c ccc@{}}
\toprule[1.2pt]
$f(x,s)$ & $s_1$ & $s_2$ & $s_3$ \\
\midrule[1.2pt]
$a$ & 9 & 1 & 5 \\
$b$ & 1 & 9 & 6 \\
$c$ & 4 & 4 & 8 \\
\bottomrule[1.2pt]
\end{tabular}
\end{center}
Decisions $a$ and $b$ each hedge against one of $s_1$, $s_2$ but are exposed to the other. Decision $c$ is balanced across $s_1$ and $s_2$ (cost~4 each) but pays the highest cost under~$s_3$.

We can compute $V(R) = \min_{x\in\{a,b,c\}} \max_{s\in R} f(x,s)$ for all subsets:
\begin{center}
\small
\begin{tabular}{@{}l c c c l@{}}
\toprule[1.2pt]
$R$ & $\max_s f(a,s)$ & $\max_s f(b,s)$ & $\max_s f(c,s)$ & $V(R)$ \\
\midrule[1.2pt]
$\{s_1\}$           & 9 & 1  & 4  & 1 \quad ($b$) \\
$\{s_2\}$           & 1 & 9  & 4  & 1 \quad ($a$) \\
$\{s_3\}$           & 5 & 6  & 8  & 5 \quad ($a$) \\
$\{s_1,s_2\}$       & 9 & 9  & 4  & 4 \quad ($c$) \\
$\{s_1,s_3\}$       & 9 & 6  & 8  & 6 \quad ($b$) \\
$\{s_2,s_3\}$       & 5 & 9  & 8  & 5 \quad ($a$) \\
$\{s_1,s_2,s_3\}$   & 9 & 9  & 8  & 8 \quad ($c$) \\
\bottomrule[1.2pt]
\end{tabular}
\end{center}

\paragraph{PRISE trace on this instance.}
PRISE (Algorithm~\ref{alg:prise_supervision_fixed2}) greedily selects the candidate maximising $V(R_t \cup \{\xi_j\})$:
\begin{itemize}[nosep,leftmargin=*]
\item \textbf{Step~0} ($R_0=\varnothing$): Evaluate $V(\{s_1\}){=}1$, $V(\{s_2\}){=}1$, $V(\{s_3\}){=}5$. Select $s_3$ (highest). Gain $\Delta_0 = 5$.
\item \textbf{Step~1} ($R_1=\{s_3\}$): Evaluate $V(\{s_3,s_1\}){=}6$, $V(\{s_3,s_2\}){=}5$. Select $s_1$. Gain $\Delta_1 = 6-5 = 1$.
\item \textbf{Step~2} ($R_2=\{s_3,s_1\}$): Evaluate $V(\{s_3,s_1,s_2\}){=}8$. Select $s_2$. Gain $\Delta_2 = 8-6 = 2$.
\end{itemize}
Note that all marginal gains are strictly positive, yet the gains are \emph{non-monotone}: $\Delta_1 = 1 < 2 = \Delta_2$. The greedy trace recovers the full-scenario objective value $V(\{s_1,s_2,s_3\}) = 8$.

\paragraph{Submodularity violation.}
The marginal gain of $s_2$ \emph{increases} from $\Delta(\{s_3\}, s_2)=0$ to $\Delta(\{s_3,s_1\}, s_2)=2$ as the base set grows from $\{s_3\}\subset\{s_3,s_1\}$, violating the submodularity condition. Intuitively, when only $s_3$ is present, decision $a$ achieves $V(\{s_3\})=5$ and $V(\{s_3,s_2\})=5$---adding $s_2$ is ``free'' because $a$ already hedges against it (cost~1). But once $s_1$ is included, $a$'s high cost under $s_1$ (cost~9) forces a switch to $b$, and $s_2$ now imposes an unavoidable cost increase (since $b$ has cost~9 under $s_2$, forcing a further switch to $c$). This interaction between scenarios through the first-stage decision makes $V$ non-submodular.
\end{remark}

\subsection{Theoretical Guarantees}
\label{subsec:prise_theo_guar}

PRISE is one-step optimal by construction: at each iteration, it selects a
candidate attaining the largest marginal increase in the reduced-scenario
objective. However, because $V$ is monotone but not necessarily submodular,
this local optimality does not imply the classical global approximation ratio
for greedy maximization. Nevertheless, the quality of the resulting decision
can be certified a posteriori without assuming submodularity or solving the
full 2RO problem to optimality.

For any feasible first-stage decision $x\in X$, define its regret as
\begin{equation}
\mathrm{Regret}(x)
:=\frac{Z(x)-V(\Xi)}{V(\Xi)}\times 100.
\label{eq:full_scenario_regret}
\end{equation}

\begin{proposition}[A posteriori regret bound]
\label{prop:prise_regret_certificate}
Let $R\subseteq\Xi$, $B_R$ be a valid lower bound on the reduced optimum
$V(R)$, and $\bar{x}_R\in X$ be a feasible first-stage decision returned
by the reduced solver. Define
\begin{equation}
U_R:=Z(\bar{x}_R)
=c^\top\bar{x}_R+\max_{\xi\in\Xi}Q(\bar{x}_R,\xi).
\end{equation}
Then
\begin{equation}
B_R\le V(\Xi)\le U_R.
\label{eq:prise_certificate_bounds}
\end{equation}
If $B_R>0$, the regret of $\bar{x}_R$ satisfies
\begin{equation}
\mathrm{Regret}(\bar{x}_R)
\;\le\;\frac{U_R-B_R}{B_R}\times 100.
\label{eq:prise_regret_certificate}
\end{equation}
\end{proposition}

\begin{proof}
Removing scenarios relaxes the full problem, and therefore
$B_R\le V(R)\le V(\Xi)$. Moreover, evaluating the feasible decision
$\bar{x}_R$ over all scenarios yields
$V(\Xi)=\min_{x\in X}Z(x)\le Z(\bar{x}_R)=U_R$, proving
Eq.~\eqref{eq:prise_certificate_bounds}. For $B_R>0$,
\begin{equation}
\mathrm{Regret}(\bar{x}_R)
=\left(\frac{U_R}{V(\Xi)}-1\right)\times 100
\;\le\;\left(\frac{U_R}{B_R}-1\right)\times 100
=\frac{U_R-B_R}{B_R}\times 100,
\end{equation}
which establishes Eq.~\eqref{eq:prise_regret_certificate}.
\end{proof}

The certificate remains available when the reduced MILP is terminated early:
the solver's certified best bound provides $B_R$, while any feasible incumbent
provides $\bar{x}_R$. Computing $U_R$ requires only fixing $\bar{x}_R$ and
evaluating its recourse cost over all scenarios.

This posteriori bound complements the pre-solve, data-dependent bound
of Goerigk and Khosravi~\cite[Corollary~11]{goerigk2023optimal}. SEL and VC satisfy the required assumptions of
nonnegative objective uncertainty and scenario-independent recourse. Hence,
once $R$ has been selected, their domination factor $\alpha(R)$ can be computed
from the scenario coefficients before solving either the reduced or full 2RO
problem, and it yields
\begin{equation}
Z(x_R^\star)\le\alpha(R)V(\Xi),
\qquad
\mathrm{Regret}(R)\le 100\times[\alpha(R)-1],
\end{equation}
where $x_R^\star$ denotes an optimal first-stage solution of the reduced problem $V(R)$ in Eq.~\eqref{eq:V_def}.

\subsection{PRISE Compression Budget and Convergence Comparison}
\label{subsec:prise_empirical_analysis}

We define the \emph{compression budget} $\hat{k}$ as the minimum number of scenarios needed for a method's selected subset $R^{(k)}$ to reach within 1\% of the full-scenario objective value:
\[
\hat{k} \;=\; \min\bigl\{k : (V(\Xi) - V(R^{(k)})) / V(\Xi) \leq 1\%\bigr\}.
\]
This is an analysis metric for Table~\ref{tab:effective_k_stats}, distinct from PRISE's internal stopping criterion (Algorithm~\ref{alg:prise_supervision_fixed2}), which uses marginal gain tolerance $\epsilon$. Reported as the compression ratio $\hat{k}/|\Xi|$, it measures how aggressively a method can shrink the scenario set without sacrificing solution quality.

PRISE compresses scenario sets much more aggressively than baselines. MaxSum and Random require substantially larger scenario budgets to reach comparable worst-case coverage. Gains are empirically non-increasing in 93--99\% of consecutive PRISE steps, validating the diminishing-returns assumption underlying the SWKL loss design (Section~\ref{sec:loss}).

\begin{table}[t]
\centering
%\scriptsize
% \setlength{\tabcolsep}{2pt}
\caption{
Compression ratio $\hat{k}/|\Xi|$ with budget $K{=}8$; N/C denotes the percentage of instances that did not converge within $K$ steps.}
\label{tab:effective_k_stats}
\vspace{5pt}
\resizebox{0.7\linewidth}{!}{
\begin{tabular}{@{}ll cc >{\centering\arraybackslash}p{0.7cm}>{\centering\arraybackslash}p{0.7cm}>{\centering\arraybackslash}p{0.7cm}>{\centering\arraybackslash}p{0.7cm} c@{}}
\toprule[1.2pt]
& & \multicolumn{2}{c}{$\hat{k}/|\Xi|$ (\%)$\,\downarrow$} & \multicolumn{4}{c}{\% instances with $\hat{k}/|\Xi|\le$ $\uparrow$} & N/C$\,\downarrow$ \\
\cmidrule(lr){3-4} \cmidrule(lr){5-8}
Problem & Method & Mean & Med. & 2\% & 4\% & 8\% & 12\% & (\%) \\
\midrule[1.2pt]
\multirow{3}{*}{SEL-20-50}
 & Random  & 15.3 & 16.0 &   1.2 &   1.6 &   5.2 &   8.0 & 89.2 \\
 & MaxSum  & 10.9 & 12.0 &  10.0 &  21.6 &  38.4 &  51.2 & 39.6 \\
 & PRISE   & \textbf{4.2} & \textbf{4.0} & \textbf{34.4} & \textbf{79.6} & \textbf{95.2} & \textbf{96.4} & \textbf{3.6} \\
\midrule
\multirow{3}{*}{VC-20-50}
 & Random  & 15.8 & 16.0 &   0.0 &   0.4 &   1.2 &   3.2 & 94.8 \\
 & MaxSum  & 13.9 & 16.0 &   0.4 &   4.4 &  13.6 &  25.2 & 59.2 \\
 & PRISE   & \textbf{7.8} & \textbf{8.0} & \textbf{1.2} & \textbf{18.8} & \textbf{71.6} & \textbf{91.6} & \textbf{4.4} \\
\bottomrule[1.2pt]
\end{tabular}
}
\end{table}

\begin{figure}[h]
    \centering
    \begin{subfigure}[t]{0.48\linewidth}
        \centering
        \includegraphics[width=\linewidth]{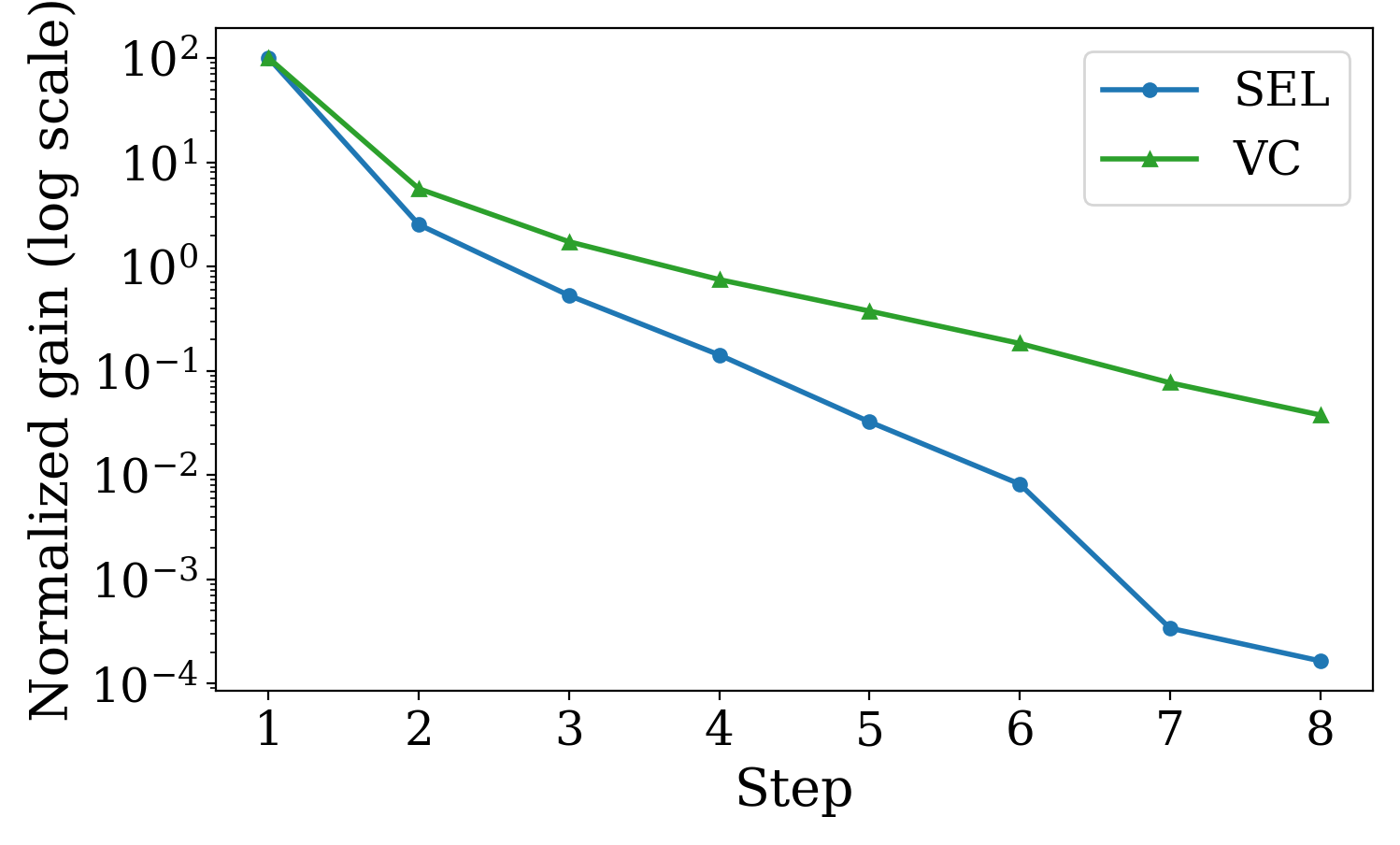}
        \caption{Mean marginal gain across iterations.}
        \label{fig:prise_marginal_gains}
    \end{subfigure}
    \hfill
    \begin{subfigure}[t]{0.48\linewidth}
        \centering
        \includegraphics[width=\linewidth]{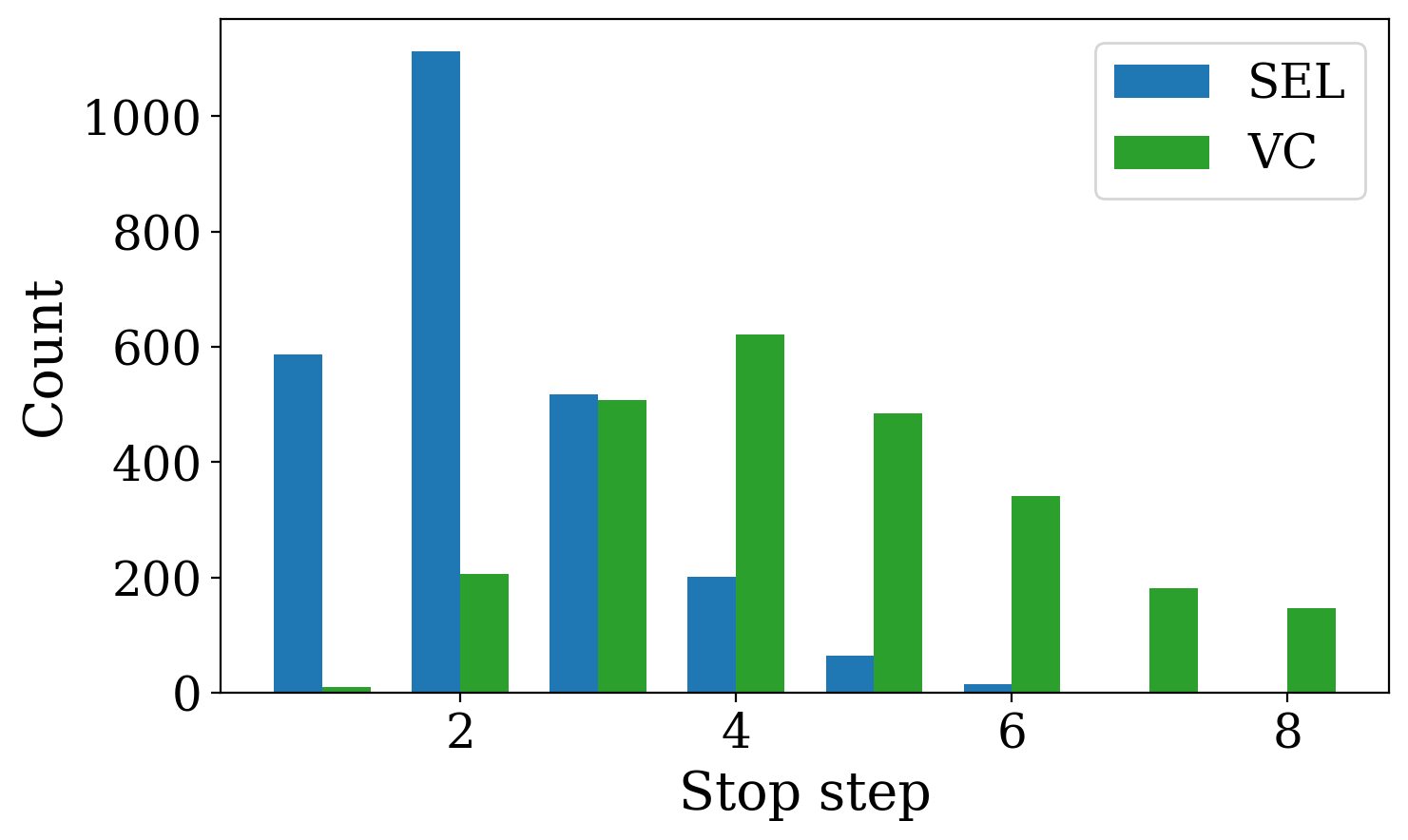}
        \caption{Distribution of PRISE stop steps.}
        %(internal stopping criterion).}
        \label{fig:prise_stop_step}
    \end{subfigure}
    \caption{Empirical analysis of PRISE on small-scale instances. (a)~Marginal gains empirically exhibit diminishing returns on average; (b)~Distribution of PRISE's stop step.}
    \label{fig:prise_analysis}
\end{figure}

\paragraph{Labeling cost.}

PRISE is not designed as a fast solver---it is a labeling oracle. At each step, PRISE evaluates the marginal gain of every remaining candidate scenario, each requiring a full MILP solve; this yields $O(\hat{k}{\cdot}S)$ solves per instance, where $\hat{k}$ is the number of selection steps. Take Table~\ref{tab:full_results_small} as an example, it reports measured PRISE wall-clock times on 250 test instances using 40-thread Gurobi (same hardware as Section~\ref{sec:experiments}). Scaling to the 2{,}000 training and 250 validation instances gives an estimated labeling cost of approximately 14.5\,wall-clock hours for both small-scale problems combined (SEL ${\approx}$0.33\,h, VC ${\approx}$14.2\,h). This is a one-time offline cost; once labels are generated, NeurPRISE training and inference require no further PRISE calls. At larger scales, PRISE labeling cost grows substantially, but the generalization results in Section~\ref{sec:gen_size} show that small-scale labels transfer effectively, avoiding the need for expensive large-scale labeling across the tested problems.

\subsection{Connection to Column-and-Constraint Generation}
\label{subsec:prise_vs_ccg}
PRISE's iterative scenario-addition structure is motivated by column-and-constraint generation (CCG)~\cite{zeng2013solving,zhao2012exact}, a standard exact method for 2RO. Both methods incrementally expand a scenario subset, but they differ fundamentally in their scenario-selection rule and algorithmic purpose:

\begin{itemize}[leftmargin=*,itemsep=0.3em]
\item \textbf{CCG} solves the restricted master problem $V(\widehat{R})$ to obtain a first-stage solution $x^\star(\widehat{R})$, then adds the scenario that is worst for this \emph{fixed} solution: $\xi^{\mathrm{CCG}} \in \arg\max_{\xi \in \Xi} Q(x^\star(\widehat{R}),\xi)$. This adversarial separation step is designed to close the optimality gap and certify convergence to the exact full-scenario value $V(\Xi)$.

\item \textbf{PRISE} instead \emph{re-solves} the reduced problem on $R_t \cup \{\xi_j\}$ and selects the scenario with the largest marginal gain $V(R_t \cup \{\xi_j\}) - V(R_t)$. This criterion measures the scenario's post-reoptimization impact, rather than its adversariality against a fixed incumbent solution.
\end{itemize}

\noindent The key distinction is therefore in the next-scenario selection logic: CCG asks ``which scenario is worst for the current solution?'' whereas PRISE asks ``which scenario pushes the reduced objective furthest after re-optimization?'' CCG is an exact procedure that converges to $V(\Xi)$, while PRISE is a heuristic that builds a compact reduced set for downstream use.

% C. Optimization Problems (MILP Reformulation, Problem Descriptions)
\section{Optimization Problems}
\label{subsec:problem_formulations}

\subsection{Deterministic-Equivalent MILP Reformulation}
\label{subsec:milp_reformulation}

The deterministic-equivalent MILP reformulation follows \cite{zeng2013solving}.
For finite discrete $\Xi$, the problem~\eqref{eq:2ro_formulation} can be reformulated by adding a copy of the recourse variables and constraints for each scenario:
\begin{equation}
\begin{aligned}
\min_{x,\, \eta,\, \{y^{(s)}\}} \quad & c^\top x + \eta \\
\text{s.t. } \quad & Ax \geq d, \\
& \eta \geq b_{\xi_s}^\top y^{(s)}, \quad \forall \xi_s \in \Xi, \\
& Gy^{(s)} \geq h - Ex - M\xi_s, \quad \forall \xi_s \in \Xi, \\
& x \in X,\; y^{(s)} \in Y,\; \eta \in \mathbb{R},
\end{aligned}
\label{eq:big_milp}
\end{equation}
where $y^{(s)}$ is the recourse decision for scenario $\xi_s$.
Given a reduced scenario set $R^{(k)} \subseteq \Xi$ with $|R^{(k)}| = k \ll S$, the \emph{reduced MILP} is obtained by replacing $\Xi$ with $R^{(k)}$ in~\eqref{eq:big_milp}.

\subsection{Problem Descriptions}
\label{subsec:problem_descriptions}

\paragraph{Selection Problem (SEL).}
Each of $n$ items has an uncertain cost; the goal is to select a fixed-size subset minimizing total cost. First-stage decisions commit to a partial selection, while the remaining choices are deferred to the second stage after costs are realized. Both stages involve binary decisions. Following \cite{goerigk2023optimal}, we fix the selection size to $\lfloor n/2\rfloor$ and sample costs uniformly from $\{1, \dots, 100\}$; 
the selection size enters as the RHS feature of the corresponding constraint node (Table~\ref{tab:feature_summary}), so the architecture accommodates other values without modification. Instances are denoted SEL-$n$-$s$.

Given $n$ items and scenario set $\Xi=\{\xi_1,\dots,\xi_S\}$, let $c_i$ be the first-stage cost and $d_i^{(s)}$ the second-stage cost of item $i$ under scenario $s$. The MILP reformulation is:
\begin{equation}
\begin{aligned}
\min_{x,\,\eta,\,\{y^{(s)}\}} \quad & \textstyle\sum_{i=1}^{n} c_i\, x_i + \eta \\
\text{s.t.}\quad
& \eta \ge \textstyle\sum_{i=1}^{n} d_i^{(s)}\, y_i^{(s)}, & \forall s\in[S],\\
& \textstyle\sum_{i=1}^{n}\bigl(x_i + y_i^{(s)}\bigr) = \lfloor n/2\rfloor, & \forall s\in[S],\\
& x_i + y_i^{(s)} \le 1, & \forall i\in[n],\;\forall s\in[S],\\
& x_i,\, y_i^{(s)} \in \{0,1\}.
\end{aligned}
\label{eq:sel_milp}
\end{equation}

\paragraph{Vertex Cover (VC).}
Given a graph $G=(V,E)$ with uncertain node costs, the goal is to select a minimum-cost subset of nodes covering every edge. In the 2RO setting, first-stage decisions commit to a partial cover, strategically reserving the remainder for the second stage; this allows the decision-maker to avoid committing to currently expensive nodes whose costs may be less critical under the realized scenario. Following \cite{goerigk2023optimal}, we generate random graphs with edge probability $10/n$ (average degree $\approx 10$) and uniform costs from $\{1,\dots,100\}$. Instances are denoted VC-$n$-$s$.

Let $c_i$ and $d_i^{(s)}$ be the first- and second-stage costs of node $i$:
\begin{equation}
\begin{aligned}
\min_{x,\,\eta,\,\{y^{(s)}\}} \quad & \textstyle\sum_{i\in V} c_i\, x_i + \eta \\
\text{s.t.}\quad
& \eta \ge \textstyle\sum_{i\in V} d_i^{(s)}\, y_i^{(s)}, & \forall s\in[S],\\
& (x_i + y_i^{(s)}) + (x_j + y_j^{(s)}) \ge 1, & \forall(i,j)\in E,\;\forall s\in[S],\\
& y_i^{(s)} \le 1 - x_i, & \forall i\in V,\;\forall s\in[S],\\
& x_i,\, y_i^{(s)} \in \{0,1\}.
\end{aligned}
\label{eq:vc_milp}
\end{equation}

% CFLP formulation moved to Appendix (sec/A_cflp_appendix.tex)

% D. NeurPRISE Architecture and Training (Encoding, Training Details, E+F)
\section{NeurPRISE Architecture and Training}
\label{sec:architecture}

\subsection{Encoding}
\label{subsec:graph_encoding}

\subsubsection{Bipartite Graph Representation}

Following \cite{gasse2019exact}, we represent each scenario-level MILP subproblem as a bipartite graph $G_j=(V_c \cup V_v, E)$, where $V_c$ denotes constraint nodes, $V_v$ denotes variable nodes, and edges connect variables to the constraints in which they appear (Fig.~\ref{fig:graph_bipartite}). While the graph topology is fixed across scenarios of the same instance, node features vary with scenario-dependent coefficients (e.g., costs, demands). We use this standard constraint--variable bipartite design directly.

\subsubsection{GINE Architecture}

We employ a Graph Isomorphism Network with edge features (GINE)~\cite{xu2018how, hu2019strategies} as the GNN backbone. The update rule for node $i$ at layer $\ell$ is:
\begin{equation}
\mathbf{g}_i^{(\ell+1)} = \mathrm{MLP}^{(\ell)}\!\left( (1+\epsilon^{(\ell)}) \cdot \mathbf{g}_i^{(\ell)} + \sum_{j \in \mathcal{N}(i)} \mathrm{ReLU}\!\left(\mathbf{g}_j^{(\ell)} + \mathbf{e}_{ji}\right) \right),
\end{equation}
where $\mathbf{e}_{ji}$ is the edge feature (constraint coefficient) and $\epsilon^{(\ell)}$ is a learnable scalar. We use two GINE layers followed by global mean pooling over the node index within each scenario graph to obtain per-scenario embeddings $\mathbf{h}_j \in \mathbb{R}^{d_s}$.

\subsubsection{Problem-Specific Feature Engineering}

% File: feature_table.tex
% Feature engineering summary for SEL, VC, CFLP

\begin{table}[ht]
\centering
\caption{Node and edge features used in the GNN encoder by problem class. ``Type ind.'' denotes one-hot node-kind or variable-role indicators.}
\label{tab:feature_summary}
\small
\vspace{5pt}
\begin{tabular}{@{}l l c l c@{}}
\toprule[1.2pt]
\textsc{Prob.} & \textsc{Node Features} & $d_{\text{node}}$ & \textsc{Edge Features} & $d_{\text{edge}}$ \\
\midrule[1.2pt]
SEL  & type ind.\,($\times 4$), obj.\ coeff, scenario cost, RHS, cosine sim. & 8 & constraint coeff. & 1 \\
VC   & type ind.\,($\times 4$), obj.\ coeff, scenario cost, RHS, cosine sim., degree & 9 & constraint coeff. & 1 \\
\bottomrule[1.2pt]
\end{tabular}
\end{table}

\paragraph{Feature Engineering.}
Variable node features include a type indicator (distinguishing $x$, $y$, and $\eta$ variables), first-stage cost $c_i$, and scenario-specific second-stage cost $d_i^{(s)}$. Constraint node features include the right-hand side (RHS) value and a \emph{cosine similarity} between the constraint's coefficient row and the objective vector, measuring how aligned each constraint is with the optimization direction. Edge features are coefficient matrix entries. For VC, variable nodes also include the \emph{graph degree} (incident edge count), capturing structural importance.
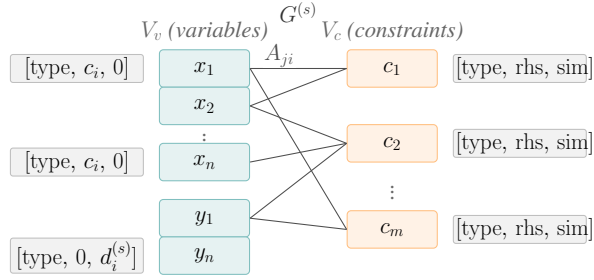
\begin{figure}[t]
\centering
\resizebox{0.55\textwidth}{!}{% Constraint-Variable Bipartite Graph (SEL/VC)
% Usage: \resizebox{\linewidth}{!}{\input{fig/graph_bipartite}}
% NOTE: No \begin{figure} wrapper — the caller provides the figure environment.
\begin{tikzpicture}[
    >=stealth,
    inputbox/.style={
        draw, rounded corners=2pt,
        minimum height=0.7cm,
        text width=3.4cm,
        align=center,
        font=\LARGE,
        inner sep=2pt,
        fill=gray!10, draw=gray!50, thick
    },
    procbox/.style={
        draw, rounded corners=3pt, minimum height=1.0cm,
        minimum width=2.4cm, align=center, font=\LARGE,
        fill=#1!10, draw=#1!50, thick
    },
    procbox/.default={blue},
    arr/.style={-, thick, color=black!70},
    lbl/.style={font=\LARGE\itshape, text=black!60},
    edgelbl/.style={font=\LARGE, text=black!70, fill=white, inner sep=1pt},
]

% Force consistent bounding box (match graph_cflp)
\useasboundingbox (-5.2, -3.0) rectangle (10.2, 4.5);

% Graph title
\node[font=\LARGE\bfseries, text=black!80] at (2.5, 4.0) {$G^{(s)}$};

% Column headers
\node[lbl, font=\LARGE\itshape] at (0, 3.5) {$V_v$ (variables)};
\node[lbl, font=\LARGE\itshape] at (5, 3.5) {$V_c$ (constraints)};

% --- Variable nodes (LEFT column) ---
\node[procbox=teal] (x1) at (0, 2.5) {$x_1$};
\node[procbox=teal] (x2) at (0, 1.5) {$x_2$};
\node[font=\LARGE, text=black!60] at (0, 0.7) {$\vdots$};
\node[procbox=teal] (xn) at (0, 0.0) {$x_n$};
\node[procbox=teal] (y1) at (0,-1.5) {$y_1$};
\node[font=\LARGE, text=black!60] at (0, -2.1) {$\vdots$};
\node[procbox=teal] (yn) at (0,-2.5) {$y_n$};

% Variable feature annotations
\node[inputbox, anchor=east] at ([xshift=-0.4cm]x1.west) {$[\mathrm{type},\, c_i,\, 0]$};
\node[inputbox, anchor=east] at ([xshift=-0.4cm]xn.west) {$[\mathrm{type},\, c_i,\, 0]$};
\node[inputbox, anchor=east] at ([xshift=-0.4cm]yn.west) {$[\mathrm{type},\, 0,\, d_i^{(s)}]$};

% --- Constraint nodes (RIGHT column) ---
\node[procbox=orange] (c1) at (5, 2.5)  {$c_1$};
\node[procbox=orange] (c2) at (5, 0.5)  {$c_2$};
\node[font=\LARGE, text=black!60] at (5, -0.7) {$\vdots$};
\node[procbox=orange] (cm) at (5,-1.8)  {$c_m$};

% Constraint feature annotations
\node[inputbox, anchor=west] at ([xshift=0.4cm]c1.east) {$[\mathrm{type},\, \mathrm{rhs},\, \mathrm{sim}]$};
\node[inputbox, anchor=west] at ([xshift=0.4cm]c2.east) {$[\mathrm{type},\, \mathrm{rhs},\, \mathrm{sim}]$};
\node[inputbox, anchor=west] at ([xshift=0.4cm]cm.east) {$[\mathrm{type},\, \mathrm{rhs},\, \mathrm{sim}]$};

% --- Edges (undirected, bipartite) ---
% Show representative connections (not all-to-all for clarity)
\draw[arr] (x1.east) -- (c1.west) node[edgelbl, pos=0.3, above] {$A_{ji}$};
\draw[arr] (x2.east) -- (c1.west);
\draw[arr] (x2.east) -- (c2.west);
\draw[arr] (xn.east) -- (c2.west);
\draw[arr] (y1.east) -- (c2.west);
\draw[arr] (x1.east) -- (cm.west);
\draw[arr] (y1.east) -- (cm.west);

\end{tikzpicture}}
\caption{Constraint--variable bipartite graph for a scenario MILP (SEL/VC), where variables form $V_v$ and constraints form $V_c$. Edges carry coefficient features $A_{ji}$.}
\label{fig:graph_bipartite}
\end{figure}

% Encoding ablation moved to CFLP appendix (sec/A_cflp_appendix.tex)

\subsection{Training Details}
\label{subsec:training_details}

\begin{minipage}[t]{0.49\textwidth}

Table~\ref{tab:hyperparameters} lists the default architecture, training, loss, and evaluation settings.
\paragraph{Input normalization.}
Before feeding data to the GNN encoder, we apply per-instance scale normalization to floating-point features. The normalization is organized by semantic feature groups: features with the same unit and decision meaning are normalized jointly, while heterogeneous features are normalized separately by their own scale. This rule removes dependence on the magnitude of each instance while preserving relative magnitudes that are meaningful for decision making.

The deterministic first-stage cost vector $c$ and the scenario-dependent second-stage cost matrix $D=[d^{(1)},\dots,d^{(S)}]$ represent costs of the same items or nodes across two decision stages. They therefore form one semantic cost group and are normalized with a shared denominator:

\end{minipage}
\hfill
\begin{minipage}[t]{0.49\textwidth}
\vspace{-8pt}
\centering\scriptsize
\captionof{table}{Default NeurPRISE hyperparameters.}
\label{tab:hyperparameters}
\begin{tabular}{@{}llr@{}}
\toprule[1.2pt]
\textsc{Component} & \textsc{Parameter} & \textsc{Value} \\
\midrule[1.2pt]
\multirow{3}{*}{GINE encoder}
 & Hidden dimension & 128 \\
 & Output dimension $d_s$ & 64 \\
 & Number of layers & 2 \\
\midrule
\multirow{4}{*}{Transformer encoder}
 & Embedding dimension $d$ & 64 \\
 & Feed-forward dimension & 128 \\
 & Number of heads & 8 \\
 & Number of layers & 2 \\
\midrule
\multirow{3}{*}{Scoring decoder}
 & Number of heads $H$ & 4 \\
 & Head dimension $d_k$ & 32 \\
 & Scoring MLP & $H \to 4H \to 1$ \\
\midrule
\multirow{6}{*}{Training}
 & Optimizer & AdamW \\
 & Learning rate & $6 \times 10^{-4}$ \\
 & Weight decay & $10^{-2}$ \\
 & Batch size & 32 \\
 & Max epochs & 200 \\
 & Early stopping patience & 10 \\
\midrule
\multirow{3}{*}{Loss (SWKL)}
 & Temperature $\tau$ & 5 \\
 & Log-compression & $\log(1 + \Delta)$ \\
 & Precision & bfloat16 (AMP) \\
\midrule
\multirow{2}{*}{Evaluation}
 & MIP gap tolerance & $10^{-4}$ \\
 & Gurobi threads & 40 \\
\bottomrule[1.2pt]
\end{tabular}
\end{minipage}

\begin{equation}
\tilde{c} = \frac{c}{\|[c;\,\mathrm{vec}(D)]\|_\infty + \epsilon}, \qquad
\tilde{D} = \frac{D}{\|[c;\,\mathrm{vec}(D)]\|_\infty + \epsilon},
\label{eq:scale_norm}
\end{equation}

where $\epsilon=10^{-8}$ avoids division by zero. The shared denominator preserves the relative magnitude between first- and second-stage costs, which is important for the model to assess scenario severity.

% Appendix E and F split into separate files for maintainability
\section{Additional Experimental Results}
\label{sec:additional_results}

\subsection{Generalization}
\label{subsec:appendix_generalization}

This section provides supplementary details for the generalization experiments in the main paper. Distribution shift setup and instance generation protocols are described below (results in Section~\ref{sec:gen_distribution}); training-set scale analysis follows in Section~\ref{subsubsec:generalization_transfer}.

\subsubsection{Distribution Shift}
\label{subsec:appendix_distribution}

\paragraph{Motivation and setup.}
In practice, the scenario distribution at deployment may differ from the training distribution. A model that overfits to the training distribution would require retraining for each new regime. We test whether NeurPRISE, trained on the uniform distributions described in Section~\ref{sec:experiments}, generalizes to structurally different scenario families without retraining. We also train an in-distribution model on each target distribution (NeurPRISE$_\text{norm}$, NeurPRISE$_\text{mm}$) to quantify the gap from not having target-distribution training data. We test robustness to two non-uniform families:

\begin{itemize}[leftmargin=*,itemsep=0.35em,topsep=0.25em]
\item \textbf{Normal.}
Per-node means $\mu_v \sim \mathrm{Uniform}[25,75]$ and standard deviations $\sigma_v = \mathrm{clip}(0.15\,\mu_v,\, 3,\, 15)$ are sampled once per problem size. Scenario costs are drawn as $\mathcal{N}(\mu_v, \sigma_v)$, clipped to $[1, 100]$. This creates a unimodal distribution shift from uniform training data.

\item \textbf{Multimodal (MM).}
Following the instance generation protocol of \cite{goerigk2024data}, we sample an instance-specific number of modes $K \sim \mathrm{Uniform}\{3,\dots,8\}$. For each mode $k \in [K]$, per-node midpoints $\mu_k \in [25,75]$ and relative deviations $\delta_k \in [0.1, 0.5]$ are drawn uniformly. Each scenario picks a mode $k$ at random and then samples node costs uniformly in $[(1{-}\delta_k)\mu_k,\,(1{+}\delta_k)\mu_k]$. This creates clustered demand regimes and tests whether the learned policy can identify important scenarios across distinct cost patterns.
\end{itemize}

Results are reported in the main paper (Section~\ref{sec:gen_distribution}, Tables~\ref{tab:distribution_generalization}--\ref{tab:distribution_mm}).

\subsubsection{Do We Need Target-Scale Training?}
\label{subsubsec:generalization_transfer}

Models trained on small instances (2{,}000 out of 2{,}500) are compared against models trained directly on the target scale (400 out of 500) in Table~\ref{tab:generalization_merged}. The small-scale model performs comparably to the target-scale model across both target scales, confirming that the architecture captures transferable structural features rather than scale-specific patterns.

A natural follow-up is whether the small-scale model's advantage comes from architectural transfer or the $5\times$ data advantage ($2{,}000$ vs.\ $400$ instances). Table~\ref{tab:data_size_ablation} isolates this factor by varying the training-set size at fixed small scale ($N \in \{500, 1500, 2500\}$, 80/10 train/val split). For SEL, regret decreases with both training-set size and scenario budget; at $k{\ge}2$, $N{=}1{,}500$ already approaches full-model quality. For VC, performance is similar across all three training-set sizes. Overall, the model is sample-efficient: $N{=}1{,}500$ instances yield competitive results for both problems, and even $N{=}500$ achieves strong generalization regret.

% File: generalization_sidebyside_appendix.tex
% Tables 12 (generalization) and 13 (data size ablation) side-by-side.

\noindent
\begin{minipage}[t]{0.49\textwidth}
\centering
\scriptsize
\setlength{\tabcolsep}{1.5pt}
\captionof{table}{Small-scale versus target-scale training comparison.}
\label{tab:generalization_merged}
\renewcommand\arraystretch{1.3}
\begin{tabular}{@{}ll cccc cccc@{}}
\toprule[1.2pt]
& & \multicolumn{4}{c}{\shortstack{\textsc{Medium} $\downarrow$}} & \multicolumn{4}{c}{\shortstack{\textsc{Large} $\downarrow$}} \\
\cmidrule(lr){3-6} \cmidrule(lr){7-10}
\textsc{Pr.} & \textsc{Model} & $k{=}1$ & $k{=}2$ & $k{=}4$ & $k{=}6$ & $k{=}1$ & $k{=}2$ & $k{=}4$ & $k{=}6$ \\
\midrule[1.2pt]
\multirow{2}{*}{SEL} & Sm. & 5.08 & \textbf{2.72} & \textbf{1.51} & \textbf{0.88} & 4.86 & \textbf{3.37} & \textbf{1.93} & \textbf{1.39} \\
 & Tgt. & \textbf{4.38} & 4.12 & 2.19 & 1.76 & \textbf{4.18} & 3.61 & 2.53 & 1.62 \\
\midrule
\multirow{2}{*}{VC} & Sm. & \textbf{15.17} & 10.06 & \textbf{5.81} & 3.98 & 15.05 & 9.09 & \textbf{5.70} & 5.09 \\
 & Tgt. & 15.26 & \textbf{9.69} & 6.50 & \textbf{3.65} & \textbf{14.92} & \textbf{9.06} & 6.85 & \textbf{4.89} \\
\bottomrule[1.2pt]
\end{tabular}
\end{minipage}
\hfill
\begin{minipage}[t]{0.49\textwidth}
\centering
\scriptsize
\setlength{\tabcolsep}{1.5pt}
\captionof{table}{Effect of training-set size on in-distribution and generalization performance.}
\label{tab:data_size_ablation}
\begin{tabular}{@{}ll cccc cccc@{}}
\toprule[1.2pt]
& & \multicolumn{4}{c}{\shortstack{In-Distribution $\downarrow$}} & \multicolumn{4}{c}{\shortstack{Generalization $\downarrow$}} \\
\cmidrule(lr){3-6} \cmidrule(lr){7-10}
\textsc{Pr.} & $N$ & $k{=}1$ & $k{=}2$ & $k{=}4$ & $k{=}6$ & $k{=}1$ & $k{=}2$ & $k{=}4$ & $k{=}6$ \\
\midrule[1.2pt]
\multirow{3}{*}{\rotatebox[origin=c]{90}{SEL}} & 500 & \textbf{3.54} & 3.08 & 2.05 & 1.41 & \textbf{4.49} & 3.85 & 1.94 & 1.68 \\
 & 1500 & 5.54 & 2.22 & 0.87 & 0.57 & 4.77 & 3.27 & 1.72 & 1.01 \\
 & 2500 & 5.59 & \textbf{2.18} & \textbf{0.79} & \textbf{0.42} & 5.08 & \textbf{2.72} & \textbf{1.51} & \textbf{0.88} \\
\midrule
\multirow{3}{*}{\rotatebox[origin=c]{90}{VC}} & 500 & 14.44 & 9.26 & 4.51 & 2.97 & \textbf{14.20} & \textbf{9.06} & 6.20 & \textbf{3.65} \\
 & 1500 & \textbf{14.42} & \textbf{8.92} & 4.41 & 2.38 & 15.55 & 10.04 & \textbf{5.59} & 3.95 \\
 & 2500 & 14.75 & 9.36 & \textbf{4.20} & \textbf{2.34} & 15.17 & 10.06 & 5.81 & 3.98 \\
\bottomrule[1.2pt]
\end{tabular}
\end{minipage}

\subsubsection{Cross-Problem Transfer}
\label{subsubsec:cross_problem_transfer}

We next investigate if a model trained on one problem class can provide a
useful initialization for another. Specifically, we consider transfer in both
directions between SEL and VC small size problem.

We consider target-problem label budgets of $1\%$, $2.5\%$, $5\%$, and
$10\%$ for fine-tuning the source-trained model. We report the performance of
the resulting source-initialized model. For a fair comparison, we also measure
the regret of a model trained from scratch using the same label budget. This
comparison allows us to quantify the transfer gain at each label budget. The
results demonstrate effective cross-problem transfer in both directions under
reduced label budgets.

\paragraph{SEL$\to$VC.}
\label{subsubsec:cross_problem_init}
A VC model is initialised from a trained SEL model by copying all weights and
zero-filling the additional input column. As shown in
Table~\ref{tab:transfer_both}, at the $5\%$ label budget, SEL initialisation
reduces regret by $18.5\%$ relative to training from scratch with the same
budget. Because generating VC labels is substantially more expensive than
generating SEL labels, pretraining on SEL provides a practical means of
reducing target-task labeling costs.

\paragraph{VC$\to$SEL.}
\label{subsubsec:vc_to_sel_backward}
To align the input dimensions, an SEL model is initialised from a trained VC
model by dropping the VC-specific degree column and copying all remaining
weights. As shown in the lower block of Table~\ref{tab:transfer_both}, VC
initialisation reduces regret at every reduced label budget, with gains ranging
from $18.1\%$ at $20$ labels to $43.8\%$ at $100$ labels. 
\par

\begin{table*}[ht]
\centering
\caption{Cross-problem transfer. The target-only model is trained from scratch using the same target
labels. Gain is $(\text{target-only}-\text{transferred})/\text{target-only}$.}
\label{tab:transfer_both}
\begin{small}
\setlength{\tabcolsep}{2pt}
\begin{tabular}{@{}ll cc c@{}}
\toprule[1.2pt]
\multirow{2}{*}{\textsc{Transfer}} & \multirow{2}{*}{Label budget}
& \multicolumn{2}{c}{Regret (\%) $\downarrow$}
& \multirow{2}{*}{Gain $\uparrow$} \\
\cmidrule(lr){3-4}
& & Transferred model & Target-only model & \\
\midrule[1.2pt]
\multirow{5}{*}{\rotatebox[origin=c]{90}{SEL$\to$VC}}
 & 1\%    & \textbf{4.65} & 5.17 & 10.1\% \\
 & 2.5\%  & \textbf{4.48} & 5.35 & 16.3\% \\
 & 5\%    & \textbf{4.21} & 5.17 & 18.5\% \\
 & 10\%   & \textbf{4.18} & 4.43 & 5.6\%  \\
 & 100\%  & \textbf{4.07} & 4.35 & 6.3\%  \\
\midrule
\multirow{5}{*}{\rotatebox[origin=c]{90}{VC$\to$SEL}}
 & 1\%    & \textbf{1.48} & 1.80 & 18.1\%  \\
 & 2.5\%  & \textbf{1.34} & 1.67 & 19.7\%  \\
 & 5\%    & \textbf{1.06} & 1.88 & 43.8\%  \\
 & 10\%   & \textbf{1.02} & 1.74 & 41.7\%  \\
 & 100\%  & 0.85 & \textbf{0.82} & \textminus3.6\%  \\
\bottomrule[1.2pt]
\end{tabular}
\end{small}

\vspace{2pt}
\end{table*}

\subsection{Seed Variance Analysis}
\label{subsec:seed_variance}

To assess sensitivity to training stochasticity, we train the default model five times with seeds $\{1, 2, 3, 4, 42\}$, keeping data and PRISE labels fixed (Table~\ref{tab:seed_variance}). Standard deviations are consistently small relative to mean regret and far smaller than the gaps between NeurPRISE and the strongest baselines, confirming that the reported advantages are robust to training randomness.

% File: seed_variance_table.tex
% Updated 2026-04-24: Gap → Regret metric.

\begin{table*}[ht]
\centering
\caption{Seed variance analysis for NeurPRISE across five training seeds.}
\label{tab:seed_variance}
\begin{small}
\setlength{\tabcolsep}{2pt}
\begin{tabular}{@{}ll cccc@{}}
\toprule[1.2pt]
\textsc{Prob.} & Seed & $k{=}1$ $\downarrow$ & $k{=}2$ $\downarrow$ & $k{=}4$ $\downarrow$ & $k{=}6$ $\downarrow$ \\
\midrule[1.2pt]
\multirow{6}{*}{\rotatebox[origin=c]{90}{SEL}}
 & 1  & 5.77 & 2.21 & 0.80 & 0.56 \\
 & 2  & 5.89 & 2.36 & 0.77 & 0.54 \\
 & 3  & 6.24 & 2.35 & 0.81 & 0.52 \\
 & 4  & 5.80 & 2.39 & 0.97 & 0.58 \\
 & 42 & 5.59 & 2.18 & 0.79 & 0.42 \\
\cmidrule(l){2-6}
 & \textbf{Avg} & \textbf{5.86{\scriptsize$\pm$0.22}} & \textbf{2.30{\scriptsize$\pm$0.09}} & \textbf{0.83{\scriptsize$\pm$0.07}} & \textbf{0.52{\scriptsize$\pm$0.06}} \\
\midrule
\multirow{6}{*}{\rotatebox[origin=c]{90}{VC}}
 & 1  & 14.42 & 9.28 & 4.50 & 2.83 \\
 & 2  & 14.03 & 8.55 & 4.40 & 2.69 \\
 & 3  & 13.89 & 8.86 & 4.11 & 2.79 \\
 & 4  & 14.11 & 8.88 & 4.31 & 2.77 \\
 & 42 & 14.75 & 9.36 & 4.20 & 2.34 \\
\cmidrule(l){2-6}
 & \textbf{Avg} & \textbf{14.24{\scriptsize$\pm$0.31}} & \textbf{8.99{\scriptsize$\pm$0.30}} & \textbf{4.31{\scriptsize$\pm$0.14}} & \textbf{2.68{\scriptsize$\pm$0.18}} \\
\bottomrule[1.2pt]
\end{tabular}
\end{small}
\end{table*}

\subsection{Time-Budgeted Exact Solver Comparison}
\label{subsec:time_budget}

Both NeurPRISE and a time-budgeted MILP are approximate solvers; we compare them under equal wall-clock budgets at medium and large scale.

%\paragraph{Setup.}
For each scenario budget $k \in \{4, 6\}$, the time-budgeted MILP on the full scenario problem receives a fixed per-instance Gurobi time limit ($1$\,s for $k{=}4$, $2$\,s for $k{=}6$), comparable to NeurPRISE's per-instance time. We omit SEL because NeurPRISE's inference time is too short to yield a meaningful solver budget. %\paragraph{Results.}
Table~\ref{tab:time_budget} shows that NeurPRISE outperforms the time-budgeted MILP in VC in medium and large scale. This confirms that NeurPRISE is most valuable when the time budget is tight relative to the full-scenario MILP.

% File: time_budget_table.tex
% Updated 2026-04-26: VC + CFLP at medium and large, k=4,6 only. Removed 5x rows.

\begin{table*}[ht]
\centering
\caption{Time-budgeted comparison. Both methods receive the same per-instance wall-clock budget, set to NeurPRISE's end-to-end time. Regret is defined in Eq.~\ref{eq:regret}.}
\label{tab:time_budget}
\begin{small}
\begin{tabular}{@{}ll c c cc@{}}
\toprule[1.2pt]
\multirow{2}{*}{\textsc{Prob.}} & \multirow{2}{*}{\textsc{Scale}}
  & \multirow{2}{*}{$k$}
  & \multirow{2}{*}{\textsc{Budget}}
  & \multicolumn{2}{c}{\textsc{Regret (\%) $\downarrow$}} \\
\cmidrule(lr){5-6}
& & & & \textsc{NeurPRISE} & \textsc{Time-budgeted MILP} \\
\midrule[1.2pt]
\multirow{4}{*}{VC}
 & Medium & 4 &  1\,s & \textbf{6.50}  & 16.65 \\
 & Medium & 6 &  2\,s & \textbf{3.65}  & 15.08 \\
 & Large  & 4 &  1\,s & \textbf{6.85}  & 50.76 \\
 & Large  & 6 &  2\,s & \textbf{4.89}  & 41.48 \\
\bottomrule[1.2pt]
\end{tabular}
\end{small}
\end{table*}

\subsection{Speed--Quality Tradeoff via Solver Tolerance}
\label{sec:solver_tolerance}

NeurPRISE's runtime is dominated by the downstream exact MILP solver, not the neural inference (Table~\ref{tab:full_results_medium}).
Since NeurPRISE already selects high-quality worst-case scenarios, the downstream solver may not need to solve the reduced MILP to full precision to obtain a good first-stage decision; we investigate whether relaxing the MIP gap tolerance yields a favorable speed--quality trade-off.
We include Neur2RO as a strong learning-based runtime baseline for reference, since it is among the fastest baselines. Figure~\ref{fig:mipgap_pareto} sweeps tolerance levels on VC at medium and large scale. NeurPRISE provides a tunable speed--quality frontier: relaxing tolerance from exact to 25\% reduces solve time while trading regret.
Beyond moderate tolerance, additional solver effort yields diminishing regret improvement: since $V(R^{(k)})$ is itself an approximation of $V(\Xi)$, solving the reduced problem to full precision does not guarantee the best first-stage decision.

\begin{figure}[ht]
\centering
\includegraphics[width=\linewidth]{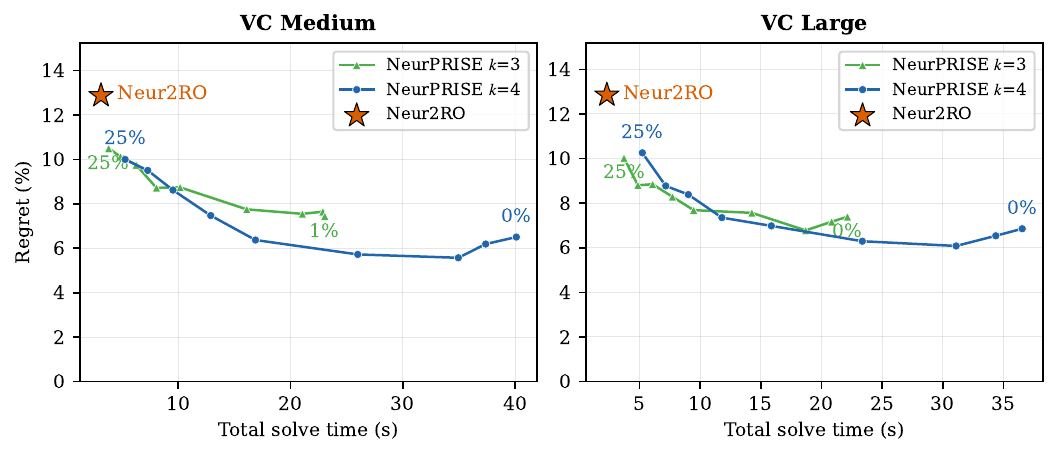}
% \vspace{-5pt}
\caption{Regret (\%) vs.\ total solve time (50 instances) at different MIP gap tolerances. 
%Neur2RO reference shown as star.
}
\label{fig:mipgap_pareto}
\end{figure}

\subsection{Scenario-Count Growth Control}
\label{subsec:runtime_scaling}

NeurPRISE controls runtime growth as the original scenario set expands. 
Although neural inference scores all $S$ scenarios and therefore scales 
approximately linearly with $S$, the downstream MILP is solved only on the 
selected $k$ scenarios.

Medium and large instances in Table~\ref{tab:full_results_medium} differ in
both instance hardness and scenario count $S$. To separate the two, we take the 50 medium
test instances (native $S{=}50$) and the 50 large test instances
(native $S{=}100$), vary $S$ by truncating each scenario set, and
fix the reduction budget at $k{=}6$. Tables~\ref{tab:scenario_runtime_med}
and~\ref{tab:scenario_runtime_lrg} report runtime ratios relative to the
smallest $S$ in each population.

NeurPRISE's end-to-end runtime has two components: (i) neural inference, which
scores all $S$ scenarios via a forward pass; and (ii) the MILP solve,
performed on the $k$-scenario reduced problem. The first component
grows with $S$, while the second is governed by $k$. Since the MILP
solve dominates at $k{=}6$, total runtime grows much more slowly than the
original scenario count.

\noindent
\begin{minipage}[t]{0.48\textwidth}
\centering
\captionof{table}{Medium instances.}
\label{tab:scenario_runtime_med}
\begin{small}
\setlength{\tabcolsep}{2pt}
\begin{tabular}{@{}r ccc@{}}
\toprule[1.2pt]
$S$ & Infer.\ & Solve & Total \\
\midrule[1.2pt]
10  & $1.0\times$ & $1.0\times$ & $1.0\times$ \\
20  & $1.1\times$ & $1.2\times$ & $1.2\times$ \\
50  & $1.2\times$ & $1.5\times$ & $1.5\times$ \\
\bottomrule[1.2pt]
\end{tabular}
\end{small}
\end{minipage}
\hfill
\begin{minipage}[t]{0.48\textwidth}
\centering
\captionof{table}{Large instances.}
\label{tab:scenario_runtime_lrg}
\begin{small}
\setlength{\tabcolsep}{2pt}
\begin{tabular}{@{}r ccc@{}}
\toprule[1.2pt]
$S$ & Infer.\ & Solve & Total \\
\midrule[1.2pt]
20  & $1.0\times$ & $1.0\times$ & $1.0\times$ \\
50  & $1.2\times$ & $1.1\times$ & $1.1\times$ \\
100 & $2.3\times$ & $1.2\times$ & $1.2\times$ \\
\bottomrule[1.2pt]
\end{tabular}
\end{small}
\end{minipage}

\vspace{0.5em}
The pattern is consistent across both instance populations. On GPU, inference grows
sub-linearly with $S$, and the MILP solve grows far more slowly than the
original scenario count. Consequently, total runtime stays well below linear
growth in $S$. This confirms that
NeurPRISE decouples the dominant downstream optimization cost from the
original scenario count, providing practical runtime growth control as the
uncertainty set expands.

\subsection{Offline cost reduction: Cheaper Labels}
\label{subsec:cheap_labels}

\emph{A natural question is whether PRISE's labeling cost can be further reduced by using approximate solutions.} We evaluate the following three relaxed PRISE configurations on medium-scale VC instances; VC is the most computationally demanding problem in our experiments. Three relaxed configurations differ from the exact labeling procedure
\emph{only} in the forward-pass inner solve: \textbf{tl2s} caps each candidate solve at $2$\,seconds
and \textbf{gap002} and \textbf{gap005} relax the MIP gap
tolerance to $0.02$ and $0.05$, respectively. Each configuration produces a complete label set for VC medium; we train
NeurPRISE on each with three seeds $\{1,2,42\}$ and evaluate on the same 50
held-out instances at $k\in\{1,2,4,6\}$.

Table~\ref{tab:cheap_label_vcmed} shows that relaxed labeling incurs only a modest increase
in regret while yielding substantial time savings, particularly under
\textbf{gap005}.
\par

\begin{table*}[ht]
\centering
\caption{Mean regret (\%, averaged over three training seeds) and labeling time savings (\%) on VC medium
for NeurPRISE with exact and relaxed PRISE labels.}
\label{tab:cheap_label_vcmed}
\begin{small}
\setlength{\tabcolsep}{6pt}
\begin{tabular}{@{}l c cccc@{}}
\toprule[1.2pt]
Method & Label time saving $\uparrow$ & $k{=}1$ $\downarrow$ & $k{=}2$ $\downarrow$ & $k{=}4$ $\downarrow$ & $k{=}6$ $\downarrow$ \\
\midrule[1.2pt]
Exact  & ---             & \textbf{15.06} & 9.46          & 6.12          & \textbf{3.49} \\
tl2s   & 15.6          & 15.46          & 9.57          & 6.36          & 3.62 \\
gap002 & 19.7          & 15.52          & 9.63          & \textbf{6.10} & 3.64 \\
gap005 & \textbf{63.7} & 15.25          & \textbf{9.41} & 6.57          & 3.73 \\
\bottomrule[1.2pt]
\end{tabular}
\end{small}
\end{table*}

\subsection{Sensitivity to PRISE Supervision}
\label{subsec:supervision_sensitivity}

To examine whether the PRISE labels remain a useful training signal when they are perturbed, we test how sensitive NeurPRISE is to such perturbations. We take the existing PRISE label file of the VC small instances and apply one perturbation to it, which creates a new perturbed label file: (i)~multiplying each marginal gain by $e^{\sigma z}$, where $z$ is standard Gaussian noise; (ii)~shuffling the selection order of the selected scenarios; or (iii)~remapping the selected scenarios among all scenarios. We then retrain NeurPRISE from each perturbed label file with the same five training seeds as Table~\ref{tab:seed_variance}. Table~\ref{tab:supervision_sensitivity} reports the results; the unperturbed row is taken from Table~\ref{tab:seed_variance}. Gain noise changes mean regret only slightly, within the variation across training seeds. Shuffling the selection order raises mean regret at $k{=}1$, while at $k{\ge}2$ it stays within the range of the unperturbed and gain-noise rows. Remapping increases mean regret at every budget, with a much larger variation across training seeds. NeurPRISE is thus robust to the gain values in PRISE labels but sensitive to which scenarios PRISE selects.

\begin{table}[ht]
\centering
\caption{NeurPRISE regret (\%) on VC small instances when trained on perturbed PRISE supervision.}
\label{tab:supervision_sensitivity}
\begin{small}
\setlength{\tabcolsep}{4pt}
\begin{tabular}{@{}l cccc@{}}
\toprule[1.2pt]
Training supervision & $k{=}1$ $\downarrow$ & $k{=}2$ $\downarrow$ & $k{=}4$ $\downarrow$ & $k{=}6$ $\downarrow$ \\
\midrule[1.2pt]
Unperturbed & 14.24{\scriptsize$\pm$0.31} & 8.99{\scriptsize$\pm$0.30} & 4.31{\scriptsize$\pm$0.14} & 2.68{\scriptsize$\pm$0.18} \\
Gain noise, $\sigma{=}0.1$ & 14.16{\scriptsize$\pm$0.20} & 8.96{\scriptsize$\pm$0.15} & 4.35{\scriptsize$\pm$0.08} & 2.70{\scriptsize$\pm$0.08} \\
Gain noise, $\sigma{=}0.7$ & 14.31{\scriptsize$\pm$0.23} & 8.90{\scriptsize$\pm$0.12} & 4.42{\scriptsize$\pm$0.15} & 2.64{\scriptsize$\pm$0.10} \\
Selection order shuffled & 15.40{\scriptsize$\pm$0.39} & 9.05{\scriptsize$\pm$0.18} & 4.42{\scriptsize$\pm$0.22} & 2.48{\scriptsize$\pm$0.10} \\
Selected scenarios remapped among all scenarios & 15.02{\scriptsize$\pm$3.43} & 11.58{\scriptsize$\pm$3.67} & 9.26{\scriptsize$\pm$3.83} & 7.73{\scriptsize$\pm$3.73} \\
\bottomrule[1.2pt]
\end{tabular}
\end{small}
\end{table}

\section{Ablation Studies}
\label{subsec:ablation}

\subsection{Encoding and Fusion Strategy}
\label{subsubsec:ablation_model_choice}

NeurPRISE first encodes each scenario graph with a shared GNN and then refines the resulting per-scenario embeddings through cross-scenario attention to obtain set-aware scenario embeddings (Section~\ref{sec:arch}). We ablate how instance context and scenario embeddings are fused in the decoder, comparing three strategies:
(i) \textbf{LF (Late Fusion):} instance context and scenario embeddings are produced by separate GNNs; fusion occurs \emph{after} cross-scenario attention, via context-to-scenario attention in the decoder;
(ii) \textbf{EF (Early Fusion):} instance context and scenario embeddings are produced by separate GNNs; the instance-context embedding is concatenated as a global token and jointly refined with scenario embeddings \emph{during} cross-scenario attention;
(iii) \textbf{Ours:} no separate instance-context encoder; the instance context is derived by mean pooling over pre-attention GNN embeddings, and scenarios are scored via multi-head scoring between this pooled context and the attention-refined scenario embeddings.

Results are in Table~\ref{tab:model_ablation_small_same_dist} (mean over 5 seeds). We adopt Ours as the default: it avoids a separate encoder for instance context, reducing model complexity without sacrificing performance.

\subsection{Loss Comparison}
\label{subsubsec:ablation_loss_comparison}

Since this task can be viewed as scenario selection, a natural baseline is per-scenario binary classification (\textbf{BCE}): multi-hot targets from the PRISE trace, treating every selected scenario equally. Our approach, \textbf{SWKL} (Score-Weighted KL), converts log-compressed marginal gains into a probability distribution via temperature-scaled softmax and trains with KL divergence, so that ordinal misrankings are penalised in proportion to gain differences.

Table~\ref{tab:loss_ablation_merged} compares the two losses at small and medium scale. At small scale SWKL is moderately better, winning more cells on VC; at medium the advantage becomes clear on SEL and VC. We adopt SWKL as the default for its broader advantage across problems.

% File: ablation_sidebyside.tex
% Tables 14 (fusion) and 15 (loss) side by side.

\begin{table*}[ht]
\begin{small}
\setlength{\tabcolsep}{2pt}

\noindent
\begin{minipage}[t]{0.44\textwidth}
\centering
\captionof{table}{Encoder-fusion ablation on small instances, averaged over 5 seeds.}
\label{tab:model_ablation_small_same_dist}
\begin{tabular}{@{}ll cccc@{}}
\toprule[1.2pt]
\multirow{2}{*}{\textsc{Prob.}} & \multirow{2}{*}{\textsc{Model}} & \multicolumn{4}{c}{Regret (\%) $\downarrow$} \\
\cmidrule(lr){3-6}
& & $k{=}1$ & $k{=}2$ & $k{=}4$ & $k{=}6$ \\
\midrule[1.2pt]
\multirow{3}{*}{\rotatebox[origin=c]{90}{SEL}} & LF & \textbf{5.31} & 2.32 & 1.05 & 0.68 \\
 & EF & 5.39 & \textbf{2.27} & 0.98 & 0.58 \\
 & Ours & 5.86 & 2.30 & \textbf{0.83} & \textbf{0.52} \\
\midrule
\multirow{3}{*}{\rotatebox[origin=c]{90}{VC}} & LF & \textbf{14.01} & \textbf{8.73} & 4.38 & 2.68 \\
 & EF & 14.07 & 8.99 & \textbf{4.31} & \textbf{2.66} \\
 & Ours & 14.24 & 8.99 & \textbf{4.31} & 2.68 \\
\bottomrule[1.2pt]
\end{tabular}
\end{minipage}
\hfill
\begin{minipage}[t]{0.54\textwidth}
\centering
\captionof{table}{Loss-function ablation.}
\label{tab:loss_ablation_merged}
\begin{tabular}{@{}ll l cccc@{}}
\toprule[1.2pt]
\textsc{Prob.} & \textsc{Scale} & \textsc{Loss} & $k{=}1$ & $k{=}2$ & $k{=}4$ & $k{=}6$ \\
\midrule[1.2pt]
\multirow{4}{*}{SEL}
 & \multirow{2}{*}{Small}  & BCE  & 6.34 & 2.44 & \textbf{0.70} & \textbf{0.16} \\
 &                         & SWKL & \textbf{5.59} & \textbf{2.18} & 0.79 & 0.42 \\
\cmidrule(l){2-7}
 & \multirow{2}{*}{Medium} & BCE  & 7.20 & 6.51 & 4.34 & 3.08 \\
 &                         & SWKL & \textbf{4.38} & \textbf{4.12} & \textbf{2.19} & \textbf{1.76} \\
\midrule
\multirow{4}{*}{VC}
 & \multirow{2}{*}{Small}  & BCE  & 15.23 & 9.46 & 4.47 & 2.38 \\
 &                         & SWKL & \textbf{14.75} & \textbf{9.36} & \textbf{4.20} & \textbf{2.34} \\
\cmidrule(l){2-7}
 & \multirow{2}{*}{Medium} & BCE  & 16.02 & 9.87 & \textbf{6.01} & 3.77 \\
 &                         & SWKL & \textbf{15.26} & \textbf{9.69} & 6.50 & \textbf{3.65} \\
\bottomrule[1.2pt]
\end{tabular}
\end{minipage}

\end{small}
\end{table*}

% F.3 Single-Pass vs Sequential Selection — removed to save space.
% Content and table preserved in table/stepwise_ablation_table.tex.
% \subsection{Single-Pass vs.\ Sequential Selection}
% \label{subsubsec:ablation_stepwise}

% E. CFLP Feasibility Analysis — moved to CFLP appendix (sec/A_cflp_appendix.tex)
% G. Extended Analysis for CFLP (encoding, training, feasibility, results, generalization)
% File: A_cflp_appendix.tex
% Dedicated CFLP appendix — consolidated from A3, A4, A5, and shared tables.

\section{Extended Analysis for CFLP}
\label{sec:cflp_appendix}

\subsection{Problem Formulation}

%\textcolor{red}{%
Among the three problem classes studied, CFLP is structurally distinct: scenario uncertainty enters through customer demands in the constraint right-hand side rather than through objective coefficients. This distinction has two important consequences. First, the SOR baseline~\cite{goerigk2023optimal}, which requires objective-coefficient uncertainty, is inapplicable to CFLP. Second, because demand extremes can exceed the total opened capacity, reduced-set recourse may be infeasible---a failure mode absent from SEL and VC, where recourse is always feasible regardless of the selected scenarios. Additionally, CFLP employs a compact problem-specific graph encoding that reduces the scenario graph from $O(mn)$ variable nodes (under the standard constraint--variable representation) to $O(m{+}n)$ facility and customer nodes, yielding significant memory and throughput gains. These characteristics motivate a dedicated analysis separate from the SEL/VC results in the main text.%
%}

We now describe the CFLP problem details. The first stage decides which facilities to open (high-stakes infrastructure investment); the second stage assigns customers to open facilities after uncertain demand is realized. The 2RO structure is especially valuable here: since only the realized demand scenario must be served, fewer facilities may suffice compared to a single-stage robust model that hedges against all scenarios simultaneously.
We adopt the instance parameter ranges of \cite{zeng2013solving} (fixed costs $f_j\!\in\![100,1000]$, capacity costs $a_j\!\in\![10,100]$, max capacity $K_j\!\in\![200,700]$, transportation costs $c_{ij}\!\in\![1,1000]$, demands $d_i\!\in\![10,500]$), using discrete scenario uncertainty in place of their budgeted polyhedral set. Let $f_j$ be the fixed cost of opening facility $j$, $a_j$ the unit capacity cost, $K_j$ the maximum capacity, $c_{ij}$ the unit transportation cost from facility $j$ to customer $i$, and $d_i^{(s)}$ the demand of customer $i$ under scenario $s$. Instances are denoted CFLP-$m$-$n$-$s$ ($m$~facilities, $n$~customers, $S$~scenarios). We use Small(CFLP-30-30-50), Medium(CFLP-70-70-50), and Large(CFLP-70-70-100) configurations:
\begin{equation}
\begin{aligned}
\min_{x,\,z,\,\eta,\,\{y^{(s)}\}} \quad & \textstyle\sum_{j=1}^{m} f_j\, x_j + \textstyle\sum_{j=1}^{m} a_j\, z_j + \eta \\
\text{s.t.}\quad
& \eta \ge \textstyle\sum_{i=1}^{n}\sum_{j=1}^{m} c_{ij}\, y_{ij}^{(s)}, & \forall s\in[S],\\
& z_j \le K_j\, x_j, & \forall j\in[m],\\
& \textstyle\sum_{i=1}^{n} y_{ij}^{(s)} \le z_j, & \forall j\in[m],\;\forall s\in[S],\\
& \textstyle\sum_{j=1}^{m} y_{ij}^{(s)} \ge d_i^{(s)}, & \forall i\in[n],\;\forall s\in[S],\\
& y_{ij}^{(s)} \ge 0,\; z_j \ge 0,\; x_j \in \{0,1\}.
\end{aligned}
\label{eq:cflp_milp}
\end{equation}
Note that the scenario-dependent term $d_i^{(s)}$ appears in the demand constraint RHS (line~3), not in the objective coefficients---this places CFLP outside the scope of \cite{goerigk2023optimal}, whose framework requires uncertainty to enter through the objective coefficients (see Section~\ref{sec:experiments}).

\subsection{Encoding and Training}

% \paragraph{CFLP (Problem-Specific Encoding).}
Following \cite{wu2022learning}, we use a problem-specific bipartite graph with facility nodes and customer nodes. Facility features include fixed cost $f_j$, capacity cost $a_j$, and maximum capacity $K_j$; customer features include scenario-specific demand $d_i^{(s)}$; and edges encode transportation costs $c_{ij}$ (Table~\ref{tab:cflp_feature_summary}). Importantly, NeurPRISE is agnostic to the choice of scenario graph encoding: different encodings can be used as long as the encoder outputs one fixed-dimensional embedding per scenario. We evaluate the CFLP-specific encoding against the standard constraint--variable encoding below.

\begin{figure}[t]
\centering
\resizebox{0.5\textwidth}{!}{% CFLP Problem-Specific Graph Encoding
% Usage: \resizebox{\linewidth}{!}{\input{fig/graph_cflp}}
% NOTE: No \begin{figure} wrapper — the caller provides the figure environment.
\begin{tikzpicture}[
    >=stealth,
    inputbox/.style={
        draw, rounded corners=2pt,
        minimum height=0.7cm,
        text width=3.4cm,
        align=center,
        font=\LARGE,
        inner sep=2pt,
        fill=gray!10, draw=gray!50, thick
    },
    procbox/.style={
        draw, rounded corners=3pt, minimum height=1.0cm,
        minimum width=2.4cm, align=center, font=\LARGE,
        fill=#1!10, draw=#1!50, thick
    },
    procbox/.default={blue},
    arr/.style={-, thick, color=black!70},
    lbl/.style={font=\LARGE\itshape, text=black!60},
    edgelbl/.style={font=\LARGE, text=black!70, fill=white, inner sep=1pt},
]

% Force consistent bounding box (match graph_bipartite)
\useasboundingbox (-5.2, -3.0) rectangle (10.2, 4.5);

% Graph title
\node[font=\LARGE\bfseries, text=black!80] at (2.5, 4.0) {$G^{(s)}_{\text{CFLP}}$};

% Column headers
\node[lbl, font=\LARGE\itshape] at (0, 3.5) {Facilities ($m$)};
\node[lbl, font=\LARGE\itshape] at (5, 3.5) {Customers ($n$)};

% --- Facility nodes (LEFT column) ---
\node[procbox=blue] (f1) at (0, 2.5)  {$f_1$};
\node[procbox=blue] (f2) at (0, 0.0)  {$f_2$};
\node[font=\LARGE, text=black!60] at (0, -1.2) {$\vdots$};
\node[procbox=blue] (fm) at (0,-2.5)  {$f_m$};

% Facility feature annotations
\node[inputbox, anchor=east] at ([xshift=-0.4cm]f1.west) {$[f_j,\, a_j,\, K_j]$};
\node[inputbox, anchor=east] at ([xshift=-0.4cm]f2.west) {$[f_j,\, a_j,\, K_j]$};
\node[inputbox, anchor=east] at ([xshift=-0.4cm]fm.west) {$[f_j,\, a_j,\, K_j]$};

% --- Customer nodes (RIGHT column) ---
\node[procbox=green!60!black] (c1) at (5, 2.5)  {$c_1$};
\node[procbox=green!60!black] (c2) at (5, 0.0)  {$c_2$};
\node[font=\LARGE, text=black!60] at (5, -1.2) {$\vdots$};
\node[procbox=green!60!black] (cn) at (5,-2.5)  {$c_n$};

% Customer feature annotations
\node[inputbox, anchor=west] at ([xshift=0.4cm]c1.east) {$[d_i^{(s)}]$};
\node[inputbox, anchor=west] at ([xshift=0.4cm]c2.east) {$[d_i^{(s)}]$};
\node[inputbox, anchor=west] at ([xshift=0.4cm]cn.east) {$[d_i^{(s)}]$};

% --- Edges (fully connected bipartite) ---
\draw[arr] (f1.east) -- (c1.west) node[edgelbl, pos=0.3, above] {$c_{ij}$};
\draw[arr] (f1.east) -- (c2.west);
\draw[arr] (f1.east) -- (cn.west);
\draw[arr] (f2.east) -- (c1.west);
\draw[arr] (f2.east) -- (c2.west);
\draw[arr] (f2.east) -- (cn.west);
\draw[arr] (fm.east) -- (c1.west);
\draw[arr] (fm.east) -- (c2.west);
\draw[arr] (fm.east) -- (cn.west);

\end{tikzpicture}}
\caption{Problem-specific graph for CFLP. }
\label{fig:graph_cflp}
\end{figure}
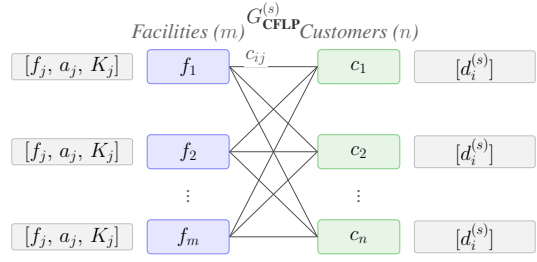

\begin{table}[ht]
\centering
\caption{CFLP graph encoding features.}
\label{tab:cflp_feature_summary}
\small
\vspace{5pt}
\begin{tabular}{@{}l l c l c@{}}
\toprule[1.2pt]
\textsc{Prob.} & \textsc{Node Features} & $d_{\text{node}}$ & \textsc{Edge Features} & $d_{\text{edge}}$ \\
\midrule[1.2pt]
CFLP & facility/customer ind.\,($\times 2$), $f_j$, $a_j$, $K_j$, $d_i^{(s)}$ & 6 & transport cost $c_{ij}$ & 1 \\
\bottomrule[1.2pt]
\end{tabular}
\end{table}

% Encoding ablation (moved wholesale from A_encoding_ablation.tex)
%\subsubsection{Graph Encoding for CFLP}
%\label{subsubsec:ablation_encoding}

For CFLP, the standard constraint--variable (CV) encoding represents assignment variables $y_{ij}^{(s)}$, producing $O(mn)$ variable nodes per scenario. In contrast, the problem-specific (PS) encoding in Section~\ref{subsec:graph_encoding} uses facility and customer nodes and reduces the graph size to $O(m+n)$.

To justify this compact encoding, we compare CV and PS on small-scale CFLP instances. Table~\ref{tab:encoding_ablation} reports regret, peak GPU memory, and training throughput. PS achieves lower regret at every budget level ($k{=}1,2,4,6$) while training $4.3{\times}$ faster per epoch with $3.6{\times}$ less GPU memory at the same batch size. Despite explicitly representing the full assignment-variable structure, CV does not improve over PS at this scale. These results suggest that PS preserves the information needed for scenario selection while avoiding the memory cost of the full CV representation.

% Encoding ablation: Problem-Specific (PS) vs Constraint-Variable (CV) for CFLP
% Updated 2026-04-23: Gap → Regret metric.

\begin{table*}[ht]
\centering
\caption{Encoding ablation on CFLP at small scale. Result is evaluated at $k{=}6$. GPU is the peak allocated memory; s/epoch is the median epoch time.}
\label{tab:encoding_ablation}
\begin{small}
\setlength{\tabcolsep}{2pt}
\begin{tabular}{lccc}
\toprule[1.2pt]
\textbf{Encoding} 
& \textbf{Regret\% ($\downarrow$)}
& \textbf{GPU (MB)}
& \textbf{s/epoch} \\
\midrule[1.2pt]
CV
& 0.16
& 10{,}303
& 23.1 \\
PS
& \textbf{0.10}
& 2{,}858
& 5.4 \\
\bottomrule[1.2pt]
\end{tabular}
\end{small}
\end{table*}

\paragraph{Training configuration.}
For CFLP instances, we reduce the batch size to fit GPU memory and scale the learning rate linearly as $\mathrm{lr}=6{\times}10^{-4}(B/32)$, where $B$ is the effective batch size. Small and medium scale problems use $B{=}24$ and $\mathrm{lr}{=}4.5{\times}10^{-4}$; Large scale problem uses $B{=}16$ and $\mathrm{lr}{=}3{\times}10^{-4}$. For the encoding ablation (Table~\ref{tab:encoding_ablation}), CV and PS use the same $B{=}24$ and learning rate.

\paragraph{Input normalization.}

All remaining floating-point features are normalized by semantic group. In CFLP, each heterogeneous feature type forms its own group, e.g., capacities, demands, fixed costs and transportation costs are normalized separately because their scales are not directly comparable. For a generic group $g$ with tensor $F_g$, we use
\begin{equation}
\tilde{F}_g = \frac{F_g}{\|F_g\|_\infty + \epsilon}.
\label{eq:generic_scale_norm}
\end{equation}
Since these features are non-negative, this maps each group to $[0,1]$ without shifting the origin.
 
\subsection{Performance and Generalization}

On CFLP, MaxSum is naturally competitive in-distribution because demand-only uncertainty makes aggregate demand a strong proxy for worst-case recourse cost. At $k=6$, NeurPRISE achieves the best regret at small scale and remains comparable to MaxSum at medium and large scales, while being much faster than PRISE (Table~\ref{tab:cflp_small}). Random/K-means are mostly infeasible because they often miss demand-extreme scenarios. Neur2RO is primarily a surrogate-based objective prediction method, not a demand-critical scenario-selection method; we adapt its auxiliary feasibility check to our settings to avoid infeasibility, and all Neur2RO CFLP results include this check.%

For generalization, we evaluate three deployment changes: larger problem size, larger scenario count, and shifted demand distributions. The size-scaling experiment increases CFLP instances by $5\times$, and the scenario-count experiment evaluates larger scenario sets with $S=100$ and $S=200$ (Table~\ref{tab:cflp_generalization}). The distribution-shift experiment considers two non-uniform demand families, both unseen during training (Table~\ref{tab:cflp_distribution_normal}). In the normal setting, per-customer demand means satisfy $\mu_j \sim \mathrm{Uniform}[180,260]$, with $\sigma_j = \mathrm{clip}(0.08\,\mu_j,\,12,\,22)$ and demands clipped to $[10,500]$. In the multimodal setting, each instance sample 3-6 demand regimes, each with per-customer centers drawn from $[80,380]$, and per-customer deviations from $[0.05,0.20]$. Each scenario selects one regime randomly and draws demands within the corresponding deviation band, with a shared per-scenario scaling factor (std deviation 8\%) to simulate global demand shifts (e.g., peak-season surges).

NeurPRISE performs consistently well under these generalization tests. Under $5\times$ size scaling and larger scenario sets, NeurPRISE obtains the lowest regret, outperforming both MaxSum and Neur2RO (Table~\ref{tab:cflp_generalization}). This suggests that the learned scoring rule captures transferable structural signals of scenario importance, rather than merely fitting demand. Training on small instances is also sufficient: the small-scale model performs comparably to, target-scale training. Under both normal and multimodal shifts, NeurPRISE trained on uniform data remains the best among methods that do not use target-distribution training data, outperforming MaxSum and Neur2RO (Table~\ref{tab:cflp_distribution_normal}). Its regret is also close to the in-distribution NeurPRISE model, indicating that uniform-trained NeurPRISE loses little from not being retrained on the target demand distribution.

% File: cflp_performance.tex
% Merged CFLP performance across all three scales (k=6 only).

\begin{table*}[ht]
\centering
\caption{CFLP performance across three scales at $k{=}6$. ---: more than 75\% infeasible.}
\label{tab:cflp_small}
\phantomsection\label{tab:cflp_medium}
\phantomsection\label{tab:cflp_large}
\begin{small}
\setlength{\tabcolsep}{2pt}
\begin{tabular}{@{}l cc cc cc@{}}
\toprule[1.2pt]
& \multicolumn{2}{c}{\textsc{Small}} & \multicolumn{2}{c}{\textsc{Medium}} & \multicolumn{2}{c}{\textsc{Large}} \\
\cmidrule(lr){2-3} \cmidrule(lr){4-5} \cmidrule(lr){6-7}
\textsc{Method} & \textsc{Reg.}$\downarrow$ & \textsc{Time} & \textsc{Reg.}$\downarrow$ & \textsc{Time} & \textsc{Reg.}$\downarrow$ & \textsc{Time} \\
\midrule[1.2pt]
Exact             & 0.00          & 151.2  & 0.0           & 395.3    & 0.0           & 836.2 \\
PRISE             & 0.00          & 1{,}623.7 & 0.0        & 1{,}820.0 & 0.0          & 6{,}157.0 \\
\cmidrule{1-7}
K-means           & ---           & ---    & ---           & ---      & ---           & --- \\
Random            & ---           & ---    & ---           & ---      & ---           & --- \\
MaxSum            & 0.15          & 28.2   & \textbf{0.1}  & 37.9     & \textbf{0.1}  & 40.3 \\
\cmidrule{1-7}
Neur2RO           & 1.25          & \textbf{23.1} & 0.9    & \textbf{16.0} & 0.9      & \textbf{34.3} \\
NeurPRISE         & \textbf{0.10} & 28.7   & \textbf{0.1}  & 40.7     & \textbf{0.1}  & 42.2 \\
\bottomrule[1.2pt]
\end{tabular}
\end{small}
\end{table*}

% File: cflp_generalization.tex
% Panels A+B merged horizontally; Panels C+D as small side-by-side tables.

\noindent
\begin{minipage}[t]{0.48\textwidth}
\centering
\begin{small}
\setlength{\tabcolsep}{2pt}
\captionof{table}{CFLP generalization.}
\label{tab:cflp_generalization}
\begin{tabular}{@{}lccc@{}}
\toprule[1.2pt]
\textsc{Method} 
& $5\times$ size 
& $S{=}100$ 
& $S{=}200$ \\
\midrule[1.2pt]
& $k{=}6$ $\downarrow$ 
& $k{=}8$ $\downarrow$ 
& $k{=}16$ $\downarrow$ \\
\midrule[1.2pt]
MaxSum    & 0.03          & 0.16          & 0.08 \\
Neur2RO   & 0.36          & 1.25          & 1.00 \\
NeurPRISE & \textbf{0.02} & \textbf{0.13} & \textbf{0.01} \\
\bottomrule[1.2pt]
\end{tabular}
\end{small}
\end{minipage}
\hfill
\begin{minipage}[t]{0.48\textwidth}
\centering
\begin{small}
\setlength{\tabcolsep}{2pt}
\captionof{table}{CFLP distribution shift at $k{=}6$.}
\label{tab:cflp_distribution_normal}
\phantomsection\label{tab:cflp_distribution_mm}
\begin{tabular}{@{}lcc@{}}
\toprule[1.2pt]
\textsc{Method} 
& \textsc{Norm.}$\,\downarrow$ 
& \textsc{MM}$\,\downarrow$ \\
\midrule[1.2pt]
MaxSum                  & 0.19          & 0.15 \\
Neur2RO                 & 0.94          & 1.88 \\
NeurPRISE$_\text{uni}$  & 0.14          & 0.13 \\
NeurPRISE$_\text{in}$   & \textbf{0.11} & \textbf{0.11} \\
\bottomrule[1.2pt]
\end{tabular}
\end{small}
\end{minipage}

\subsection{Feasibility Analysis}
\label{sec:cflp_feasibility}
\begin{minipage}[t]{0.56\textwidth}
Unlike SEL and VC, where uncertainty enters only objective coefficients and recourse is always feasible, CFLP has demand uncertainty in the constraint right-hand side (Eq.~\ref{eq:cflp_milp}). A first-stage decision $x^{(k)\star}$ optimized for a reduced set $R^{(k)}$ may open insufficient facility capacity to serve full demand under scenarios outside $R^{(k)}$, causing the recourse infeasible. Table~\ref{tab:cflp_infeasibility} reports infeasibility rates (\% of test instances where at least one scenario in $\Xi$ causes infeasible recourse) for each method and scenario budget across three problem scales. K-Means and Random exhibit high infeasibility because they miss demand extremes: K-Means selects near-centroid scenarios, while Random hits max-demand scenarios with probability ${\approx}k/S$. MaxSum is feasible in this case: with a dense assignment, the recourse feasibility of Eq.~\ref{eq:cflp_milp} reduces to total demand not exceeding total opened capacity, and MaxSum always selects the scenario with the largest total demand. As mentioned above, Neur2RO applies its feasibility check during CCG, which maintains zero infeasibility here.
\end{minipage}
\hfill
\begin{minipage}[t]{0.4\textwidth}
\centering\scriptsize
\vspace{-6pt}
\captionof{table}{CFLP infeasibility rate (\%).}
\label{tab:cflp_infeasibility}
\setlength{\tabcolsep}{1.5pt}
\resizebox{\linewidth}{!}{%
\begin{tabular}{@{}llcccc@{}}
\toprule[1.2pt]
\textsc{Scale} & \textsc{Method} & $k{=}1$ $\downarrow$ & $k{=}2$ $\downarrow$ & $k{=}4$ $\downarrow$ & $k{=}6$ $\downarrow$ \\
\midrule[1.2pt]
\multirow{5}{*}{\rotatebox[origin=c]{90}{Small}}
 & K-Means         & 99.2 & 99.2 & 97.2 & 93.6 \\
 & Random          & 95.2 & 94.0 & 91.6 & 88.8 \\
 & MaxSum          & \textbf{0.0} & \textbf{-} & \textbf{-} & \textbf{-} \\
 & Neur2RO         & \textbf{0.0} & \textbf{-} & \textbf{-} & \textbf{-} \\
\cmidrule(l){2-6}
 & NeurPRISE       & 21.6 & 5.6 & 0.4 & \textbf{0.0} \\
\midrule
\multirow{5}{*}{\rotatebox[origin=c]{90}{Medium}}
 & K-Means   & 100.0 & 98.0 & 98.0 & 86.0 \\
 & Random    & 98.0  & 92.0 & 86.0 & 86.0 \\
 & MaxSum    & \textbf{0.0} & \textbf{-} & \textbf{-} & \textbf{-} \\
 & Neur2RO   & \textbf{0.0} & \textbf{-} & \textbf{-} & \textbf{-} \\
\cmidrule(l){2-6}
 & NeurPRISE & 10.0 & \textbf{0.0} & \textbf{0.0} & \textbf{0.0} \\
\midrule
\multirow{5}{*}{\rotatebox[origin=c]{90}{Large}}
 & K-Means   & 100.0 & 100.0 & 100.0 & 100.0 \\
 & Random    & 98.0  & 98.0  & 94.0 & 92.0 \\
 & MaxSum    & \textbf{0.0} & \textbf{-} & \textbf{-} & \textbf{-} \\
 & Neur2RO   & \textbf{0.0} & \textbf{-} & \textbf{-} & \textbf{-} \\
\cmidrule(l){2-6}
 & NeurPRISE & 10.0 & 2.0 & \textbf{0.0} & \textbf{0.0} \\
\bottomrule[1.2pt]
\end{tabular}%
}
\end{minipage}

NeurPRISE recovers feasibility through scenario selection alone without any explicit feasibility check: its infeasibility rate drops monotonically with $k$ and reaches 0\% at $k{=}6$ across all three scales. This is because demand-extreme scenarios carry large marginal gains in the PRISE labeling, so they rank high in the learned scoring and are included in the selected set once $k$ is sufficient.

\stopcontents[appendix]

\end{document}